\documentclass[12pt]{article}
\usepackage{graphicx}

\usepackage{arxiv}
\usepackage{tikz}
\usepackage{comment}
\usepackage{amsmath,amssymb} 
\usepackage{color}
\usepackage{enumitem}
\usepackage{algorithm}
\usepackage{algorithmic}
\usepackage{amsthm}

\usepackage{microtype}
\usepackage{subfigure}
\usepackage{booktabs} 
\usepackage{enumitem} 
\usepackage{pgfplots}
\usepackage{standalone}
\usepackage{pgfplotstable}
\usepackage{booktabs}
\usepackage{multirow}
\usepackage[auth-lg, affil-sl]{authblk}
\usetikzlibrary{pgfplots.groupplots}
\pgfplotsset{compat=1.16}

\pgfplotsset{
every x tick label/.append style = {font=\Large},
every y tick label/.append style = {font=\Large},
}

\usepackage{amsmath,amsfonts,bm}

\def\Figref#1{Fig.~\ref{#1}}

\def\Secref#1{Sec.~\ref{#1}}

\def\eqref#1{eq.~\ref{#1}}
\def\Eqref#1{Eq.~\ref{#1}}

\def\Algref#1{Alg.~\ref{#1}}

\def\Tableref#1{Table~\ref{#1}}

\def\1{\bm{1}}

\newcommand{\test}{\mathcal{D_{\mathrm{test}}}}

\def\rvc{{\mathbf{c}}}

\def\rvt{{\mathbf{t}}}

\def\rvx{{\mathbf{x}}}
\def\rvy{{\mathbf{y}}}
\def\rvz{{\mathbf{z}}}

\def\rmI{{\mathbf{I}}}

\DeclareMathAlphabet{\mathsfit}{\encodingdefault}{\sfdefault}{m}{sl}
\SetMathAlphabet{\mathsfit}{bold}{\encodingdefault}{\sfdefault}{bx}{n}

\def\gN{{\mathcal{N}}}

\def\gP{{\mathcal{P}}}

\def\gX{{\mathcal{X}}}

\def\sP{{\mathbb{P}}}

\newcommand{\E}{\mathbb{E}}

\newcommand{\R}{\mathbb{R}}

\DeclareMathOperator*{\argmax}{arg\,max}

\theoremstyle{plain}
\newtheorem{theorem}{\protect\theoremname}
\theoremstyle{remark}

\theoremstyle{definition}

\theoremstyle{plain}
\newtheorem{corollary}{\protect\corollaryname}

\providecommand{\corollaryname}{Corollary}
\providecommand{\lemmaname}{Lemma}

\providecommand{\remarkname}{Remark}
\providecommand{\theoremname}{Theorem}

\newcommand{\etal}{\emph{et~al.}}

\usepackage[accsupp]{axessibility}  

\usepackage{hyperref}

\begin{document}

\title{Smooth-Reduce: Leveraging Patches for Improved Certified Robustness} 

\author[1]{Ameya Joshi \thanks{Work partially done during an internship at Bosch Center for AI (BCAI), Pittsburgh, PA. The project page can be found at \url{https://nyu-dice-lab.github.io/SmoothReduce/}.}}
\author[1]{Minh Pham}
\author[1]{Minsu Cho}
\author[2]{Leonid Boytsov}
\author[2]{Filipe Condessa}
\author[2]{J. Zico Kolter}
\author[1]{Chinmay Hegde}
\affil[1]{New York University}
\affil[2]{Bosch Center for AI}
\affil[ ]{\texttt{\{ameya.joshi, mp5847, mc8065, chinmay.h\}@nyu.edu}}
\affil[ ]{\texttt{\{leonid.boystov, filipe.condessa\}@us.bosch.com}}
\affil[ ]{\texttt{jkolter@cs.cmu.edu}}

\maketitle

\begin{abstract}
       Randomized smoothing (RS) has been shown to be a fast, scalable technique for certifying the robustness of deep neural network classifiers. However, methods based on RS require augmenting data with large amounts of noise, which leads to significant drops in accuracy. We propose a training-free, modified smoothing approach, Smooth-Reduce, that leverages patching and aggregation to provide improved classifier certificates. Our algorithm classifies overlapping patches extracted from an input image, and aggregates the predicted logits to certify a larger radius around the input. We study two aggregation schemes --- max and mean --- and show that both approaches provide better certificates in terms of certified accuracy, average certified radii and abstention rates as compared to concurrent approaches. We also provide theoretical guarantees for such certificates, and empirically show significant improvements over other randomized smoothing methods that require expensive retraining. Further, we extend our approach to videos and provide meaningful certificates for video classifiers. 
\end{abstract}

\section{Introduction}
\label{sec:intro}

\noindent\textbf{Motivation.} Deep networks have been shown to be notoriously prone to `` attacks'' if an adversary were allowed to modify their input~\cite{Goodfellow2018existence,szegedy2013intriguing,Athalye2018obfuscated}. 
While several heuristic ``defenses'' for such attacks have been proposed~\cite{madry2018towards,trades,mart}, 
only a handful of them are \emph{provably accurate}~\cite{wong2018provable,cohen2019certified,salman2019provably}, i.e., they provide guarantees for robust performance. The general approach in such defenses is to \emph{certify} that a deep classifier, for any input data point in a volume (parameterized by a radius) around a given input $\rvx$, does not change its predictions. 

In this line of work, Wong and  Kolter~\cite{wong2018provable} pioneered the use of bound propagation to derive upper bounds on the certification radius for networks with ReLU activations; however, this approach fails to scale to larger networks, and the bounds become vacuous quite quickly. Subsequently,  Cohen~\etal~\cite{cohen2019certified} and Salman~\etal~\cite{salman2019provably} employed \emph{randomized smoothing} (RS) to establish bounds on the (local) Lipschitz constant of a {smoothed} deep classifier. Such approaches have since been fruitfully developed to provide non-vacuous certificates. 

Randomized smoothing (RS) methods typically involve \emph{convolving} the deep classifier under consideration, $f$, with any smooth, continuous probability distribution, $\gP$ and deriving a radius of certification $R$ for all points around $\rvx$ measured in some norm. For simplicity, consider certifying models in terms of $\ell_2$-bounded input perturbations. Then, a randomized smoothing scheme produces a (bound on) a certificate parameter $R$ such that $$\sP\left(\gP \star f(\rvx) \neq \gP \star f(\rvx + \delta)\right) \approx 0 ~~ \text{for any} ~~ \| \delta \|_2 \leq R.$$ 

In practice, this type of functional convolution is achieved by randomly sampling noise vectors $\rvz_i \sim \gP$, adding them to copies of the input $\rvx$, and performing inference over each copy. The resultant smooth classifier estimates the empirical probability mass, $p_A$ for the correct class, $A$. Yang~\etal~\cite{Yang2020RandomizedSO} derive the radius of certification  using $$R = \int_{1-p_A}^{1/2} \frac{1}{\Phi(p_A)}dp_A,$$ where $p_A$ is the probability of the correct class under the noisy inference, and $\Phi(\cdot)$ is the appropriate CDF. Here, the geometry of the $\ell_p$ ball influences the choice of the noise distribution. For example, Gaussian noise provides $\ell_2$ certificates. 

\begin{figure*}[!t]
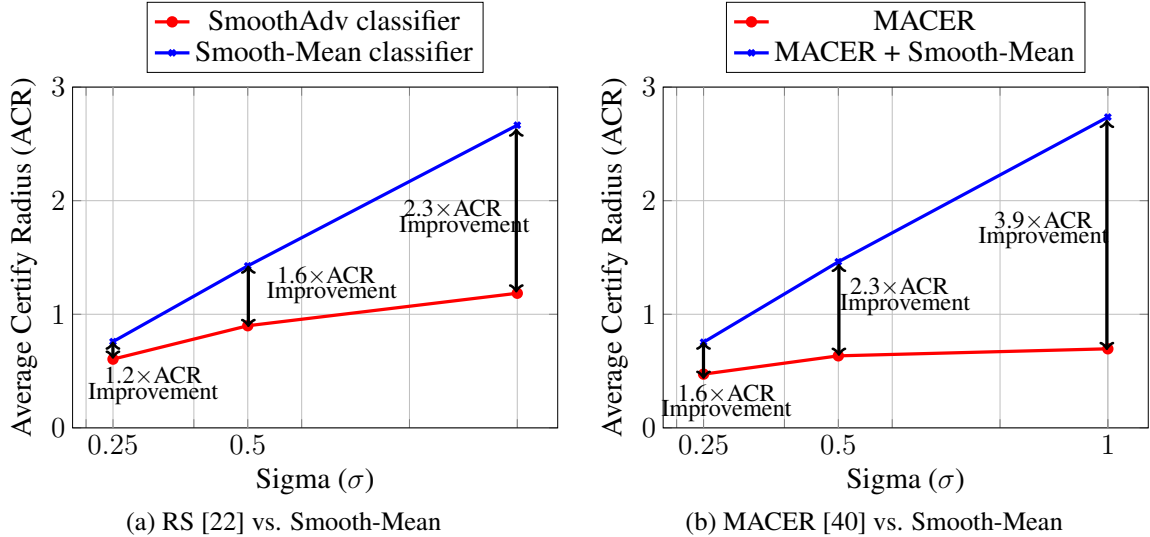

    \centering
    \begin{tabular}{c c}
         \multicolumn{1}{l}{\includegraphics[width=0.45\linewidth]{Plots/improve_mean.tex}} & \multicolumn{1}{r}{\includegraphics[width=0.45\linewidth]{Plots/improve_mean_macer.tex}} \\
         (a) RS~\cite{salman2019provably} vs. Smooth-Mean & (b) MACER~\cite{Zhai2020MACERAA} vs. Smooth-Mean
    \end{tabular}
    \caption{\textbf{Smooth-Reduce improves upon Randomized Smoothing(RS)}. Smooth-Reduce leverages patching to emulate ensembles to reduce variance of smooth predictions by the base classifier. We study two flavors of Smooth-Reduce that use \emph{max} and \emph{mean} aggregation schemes respectively. Smooth-Reduce takes any base classifier that is trained to be robust to noise and uses a RS-inspired certification algorithm to generate larger certificates with lower variants. Our approach shows significant improvements over concurrent smoothing methods in certified accuracy and abstention rates across several datasets and classifiers.}
    \label{fig:my_label}
\end{figure*}

The RS approach allows us to get non-trivial certificates for high-dimensional inputs, providing the first known family of theoretically provable defenses of deep neural network classifiers to adversarial attacks. Nevertheless, there still remain several real-world shortcomings. First, in order for the empirical probability mass to be accurately estimated, a large number of samples are necessary.
Second, the certified accuracy achieved via randomized smoothing (RS) is substantially lower than empirical accuracy achieved via heuristics such as adversarial training.
Third, the addition of noise to the inputs often significantly degrades the performance of the network. The last problem is particularly challenging, since it requires some care to handle. Typical RS methods (such as Salman~\etal\  \cite{salman2019provably} and Cohen~\etal\  \cite{cohen2019certified}) propose noise-augmented training strategy to sidestep this problem. However, in practice we see that noise-augmented training comes with a price: the certified accuracies drop off significantly as the radius increases. Several other approaches~\cite{Zhai2020MACERAA,Alfarra2020DataDR,Horvath2021BoostingRS} propose improvements to prevent such a dramatic drop-off, but they involve careful model re-training with noise augmentation, often involving several heuristic parameters. 

How then can we get better certificates? As a starting point, observe that any RS scheme involves two basic components: the base classifier $f$ and the noise distribution $\gP$. Adding noise provides certified robustness, but decreases accuracy; resolving this tradeoff is the key. Works such as~\cite{Alfarra2020DataDR,sukenik2021intriguing,Yang2020RandomizedSO} focus on $\gP$, and propose convolution with more sophisticated (sometimes even data-dependent) noise distributions. On the flip side, works such as \cite{salman2019provably,Zhai2020MACERAA,addepalli2021boosting,Horvath2021BoostingRS,Wang2021ImprovingAR,jeong2021smoothmix} focus on training better base classifiers $f$. We pursue the latter approach in this paper. 

At the heart of our approach is a simple technique that is ubiquitous in machine learning inference: \emph{ensembling}. Aggregating results from an ensemble of \emph{diverse} classifiers acting on a given data point has long been used to improve classifier performance in standard (non-adversarial) inference settings. However, in practice, training large (and diverse) ensembles for deep networks can be non-trivial (and sometimes even prohibitively expensive). The difficulty compounds when noise augmentation and adversarial training are considered. 

We overcome this difficulty by \emph{emulating} an ensemble classifier by extracting a set of (large) patches from a given image, running a (single) base classifier on all these patches, and aggregating the results. This technique has also been successfully employed by a recent series of patch-level models~\cite{Dosovitskiy2021AnII,Trockman2022PatchesAA}. Specifically, we posit that small affine transformations of an image induce sufficient diversity leading to more robust performance. Further, we also study two popular aggregation schemes for ensembling --- \emph{max} and \emph{mean} aggregation --- and demonstrate that both significantly outperform all existing RS approaches.

\noindent\textbf{Contributions.} In this paper, we propose an adaptive, ensembling-based \emph{training-free} smoothed classifier that significantly outperforms existing RS methods. 

\begin{figure*}[!t]
  \centering
  \includegraphics[width=0.75\linewidth]{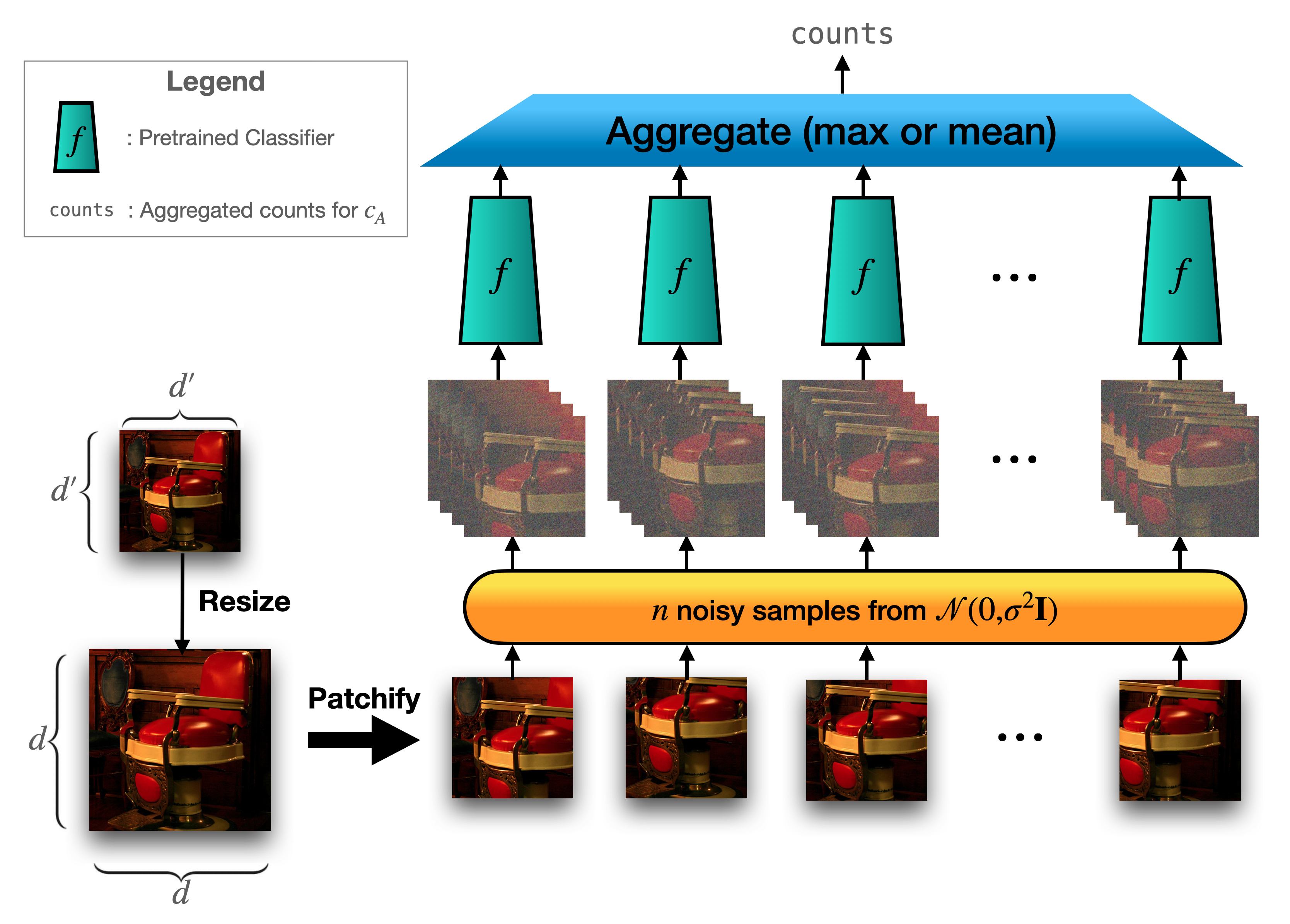}
    \caption{\textbf{Smooth-Reduce Certification.}\small \sl Smooth-Reduce modifies the RS certification in two ways. First, an input set is created to simulate an ensemble. In this case, we use patches sampled from the resized image. Following the \textsc{CERTIFY} subroutine from \cite{cohen2019certified}, noise is added to every element in the set. Next, the counts of predicted classes are aggregated to estimate $\underline{p_A}$, the probability of the most probable class, $c_A$. The final step uses \Eqref{eq:rs_smoothreduce} with $\underline{p_A}$ to derive a certificate that holds with high probability.}
    \label{fig:block}
\end{figure*}

Our specific contributions are as follows:
\begin{enumerate}
\item We present a modified smooth classifier that leverages an input set constructed by extracting patches of the input image, and achieves a higher certified radius using aggregation. 
\item We show that our certificates hold with high probability with intuitive extensions of the theoretical analysis by Cohen~\etal~\cite{cohen2019certified} and Salman \etal~\cite{salman2019provably}.
\item We demonstrate significant improvements in certification performance for CIFAR-10 and ImageNet compared with several state-of-the-art randomized smoothing approaches.
\item Finally, we extend our approach to provide certificates for video classifiers on UCF-101, therefore demonstrating that our approach scales to high dimensional domains.
\end{enumerate}

\noindent\textbf{Techniques.} Our approach consists of four steps: (1) We emulate a diverse set of inputs from a given (single) image or video input. For our image experiments, we simply sample overlapping contiguous patches. For our video experiments, we sample overlapping subvideos from the original video stream. (2) We then follow the standard randomized smoothing approach by creating $n$ copies of the input set, and adding independent noise vectors to each element in all copies  (3) We then estimate the predicted probability of each class for each copy of the input set, and take the average of the maximum estimated predicted probability for each copy. (4) Finally, we record certificates for the input using the expression in Corollary~\ref{thm:Smooth-Max} below. We explain each of these steps in detail in Section~\ref{sec:cors}.

\section{Related Work}
\label{sec:rel_work}

\noindent\textbf{Certified Defenses.} Ever since deep classifiers have been found to be vulnerable to adversarial attacks~\cite{szegedy2013intriguing,Athalye2018obfuscated,Carlini2017cwl2}, considerable efforts have been directed towards developing reliable defenses~\cite{madry2018towards,trades,wang2019improving,Samangouei2018DefenseGANPC,yin2020defense}. 
The above approaches lack strong theoretical guarantees. 
Provable defenses (that provide certificates of correctness) fall into two major categories. The first category involves establishing upper bounds on the perturbation radii for the inputs of each layer (using linear, quadratic, convex, or even mixed-integer programming) and propagating these bounds to achieve a certificate for an entire network. These include works such as ~\cite{wong2018provable,Raghunathan2018CertifiedDA,Raghunathan2018semidefinite,tjeng2017evaluating,katz2017reluplex,katz2017towards,carlini2017groundtruth,huang2017safety,Weng2018TowardsFC}. However, such approaches are computationally very expensive and do not scale at all to large, modern deep network classifiers.

\noindent\textbf{Randomized Smoothing.} The second category of provable defense involves some variation of randomized smoothing (RS), which advocate ``smoothing'' the outputs of non-linear, non-Lipschitz networks by their functional convolution with specially-chosen noise distributions. Early works such as 
\cite{lecuyer2019certified,cohen2019certified,salman2019provably} provide $\ell_2$ robustness certificates by adding Gaussian noise. Subseqent works \cite{teng2019ell_1,pmlr-v108-levine21a} have presented certificates for the $\ell_1$ and Wasserstein metrics respectively using Laplacian smoothing. Yang~\etal\cite{Yang2020RandomizedSO} provide a general approach to selecting distributions for various classes of adversarial attacks; unfortunately, certificates other than the $\ell_2$-norm have $\Omega(d^{-1/2})$ dependence, leading to trivial certificates for high dimensional inputs. 

On the practical side, the above RS methods still fall short of heuristic empirical defense methods when evaluated in terms of robust accuracy. Therefore, several works have propose modifications to the certification scheme to improve performance. MACER~\cite{Zhai2020MACERAA} maximizes surrogates of the certified radius to train better certifiable models, while \cite{Alfarra2020DataDR} finetune the variance of noise for each input data point. \cite{jeong2021smoothmix} uses an adversarial version of MixUp~\cite{Zhang2018mixupBE} to train models with better tradeoffs on accuracy and certifiable robustness. Notice, however, that all these approaches involve re-training large-scale models with different objectives and data augmentation schemes. 

\noindent\textbf{Ensembled Defenses.} Ensembling is one of the primary motivations for our Smooth-Reduce method. 
We discuss some recent relevant work in that context. Horvath~\etal\ ~\cite{Horvath2021BoostingRS} propose ensembling over diverse classifiers and show that this decreases variance of predictions, allowing better certificates. \mbox{Yang~\etal\ \cite{Yang2021OnTC}} prove that diversified gradients and large confidence margins are necessary and sufficient conditions for robust ensembles. Liu~\etal\ ~\cite{Liu2020EnhancingCR} propose a weighted ensemble of networks as the base classifier and demonstrate they provide better certificates. While all these approaches rely on model-level ensembling, we emulate ensembles by using patching and basic linear operations. This allows us to ensure diversity as well as improved base classifier performance. We also demonstrate that our approach outperforms ensemble smoothing by a large margin. 

\section{The Smooth-Reduce Framework}
\label{sec:cors}

\noindent\textbf{Preliminaries: } Let $\rvx \in \R^d$ be a given input. For ease of exposition, we suppose that $\rvx$ is an image (and extend the framework to video inputs later below). Let $f:\R^d\to[0, 1]^c$ be any classifier that takes the input and assigns each class label $c$ with probability $f_c$. Cohen~\etal~\cite{cohen2019certified} propose performing inference using the ``smooth'' classifier:
\begin{equation}
    \hat{f} =  \argmax_c \E_{\rvz \sim \gN(0, \sigma^2\rmI)} \left[ f_c(\rvx + \rvz)\right],
    \label{eq:cohen_smooth}
\end{equation}
which enjoys the benefits of guarantees of correctness. To calculate these guarantees, the standard certification approach estimates the (most probable) class $c_A \in [C]$ and the second most probable class $c_B \in [C]$, as predicted by $\hat{f}$. It also estimates upper and lower bounds (respectively), $\underline{p_A}, \overline{p_B}$, on the corresponding class probabilities. To do so, we create $n_0$ copies of the input, add $n_0$ i.i.d. Gaussian noise vectors sampled from $\gN(0, \sigma^2)$ and estimate $p_A$. 
The certified radius is then derived using the relation:
\begin{equation}
    R = \frac{\sigma}{2}\left( \Phi^{-1}(\underline{p_A}) - \Phi^{-1}(\overline{p_B})\right),
    \label{eq:smooth_radius}
\end{equation}
where $\Phi^{-1}(\cdot)$ is the inverse Gaussian CDF (see \cite{cohen2019certified,salman2019provably} for a rigorous derivation). Notice that the above procedure makes no assumptions on the base classifier $f$, and can be used to achieve a certified radius for any model (including deep neural network classifiers). 

According to \Eqref{eq:smooth_radius}, in order to obtain a higher radius of certification, we can increase either the variance of the noise or the estimated probability of the true class. However, adding large amounts of noise to the input leads to degradation in the performance of $\hat{f}$ (compared to $f$), and could give poor classification performance. Indeed, the majority of works focus on training deep classifiers $f$ that are robust to noisy inputs. Instead of pursuing this path, we focus on obtaining an improved estimate of $\underline{p_A}$.

\subsection{Smooth-Reduce}

A general approach to obtaining high-quality predictions is via \emph{ensembling}. Using an ensemble of classifiers tends to decrease the variance of the predicted probabilities while improving accuracy~\cite[P. 256]{Goodfellow-et-al-2016}. However, deep networks are very expensive to train, and training a large (and diverse) set of deep classifiers for a given training dataset can be prohibitive. This challenge is exacerbated in RS approaches which tend to require re-training models with noise augmentation.

Instead, we draw inspiration from the folklore practice of using \emph{cropping} during inference to improve performance, as well as recent empirical observations regarding the considerable effectiveness of patch-based classification~\cite{Dosovitskiy2021AnII,Trockman2022PatchesAA}. We propose patching as a mechanism to create a diverse set of (sub)images from a single input image; this allows us to emulate an ensemble while using a single base classifier. 


Our method works as follows. We create a set of inputs, $\gX = \{ \rvx_1, \rvx_2 \dots \rvx_k \}$ from a given input (base) image, $\rvx$, by using sampling (uniformly at random) sub-images with $d'$ total pixels (with $d' < d$) and upsampling each sub-image to the original resolution (with $d$ pixels). All base images and patches are assumed to be square for simplicity. We then define a modified version of the smooth classifier from \Eqref{eq:cohen_smooth} that we call the ``Smooth-Reduce'' classifier:
\begin{eqnarray}
    \bar{f}(\rvx) &= \argmax_c \E_{\rvz \sim \gN(0, \sigma^2\rmI)}~          \textsc{Aggregate}_{i=1}^k(f_c(\rvx_i + \rvz) ) \, ,
    \label{eq:max_rs}
\end{eqnarray}
where $\textsc{Aggregate}_{i=1}^k$ is a routine that \emph{reduces} (combines) the predicted logits for inputs enumerated over the set $\gX$. For our approach, we consider two specific aggregation functions, \emph{max} and \emph{mean} over the predicted logits. 

The Smooth-Reduce classifier, $\bar{f}$, is a simple modification of the standard RS approach. However, we find that it improves over $\hat{f}$ in two important aspects. Firstly, since it emulates an ensemble of classifiers, the variance of the estimated probability $\underline{p_A}$ is reduced, leading to sharper bounds on $\underline{p_A}$. (For a more in-depth discussion, see also the Appendix and \cite{Horvath2021BoostingRS}) Further, we find that it also increases the estimated probability values $p_A$ themselves; both aggregation options in Smooth-Reduce lead to more confident classification probabilities than the base classifier $f$. For this to hold, we have to ensure that the patches are large enough (so that meaningful classification is achieved), and that we extract sufficiently many patches from the input image (so that we get boosts via aggregation). 


\subsection{Theoretical Analysis}

To derive certificates of performance for our proposed Smooth-Reduce classifier $\bar{f}$, we need to rethink \Eqref{eq:smooth_radius} when used with the new class probabilities. We first restate:

\begin{theorem}[taken from \cite{cohen2019certified}]
    Let $c_A, c_B \in [C]$ be the most likely and second-most likely classes, and $\underline{p_A}, \overline{p_B} \in [0,1]$ be the probability estimates associated with $c_A$ and $c_B$. If 
    \begin{equation*}
        \sP_{\rvz} \left(f(\rvx + \rvz) = c_A\right) \geq \underline{p}_A \geq \bar{p_B} \geq \max_{c \neq c_A} \sP_{\rvz}\left( f(\rvx + \rvz)=c\right),  
    \end{equation*}
    then $\hat{f}(\rvx+\delta) = c_A$ for all vectors $\delta$ satisfying $\|\delta\|_2 \leq R$, where \[R = \frac{\sigma_{\rvz}}{2}\left(\Phi^{-1}(\underline{p}_A) - \Phi^{-1}(\bar{p}_B)\right).\]
\end{theorem}

The above derivation for the certified radius for $\hat{f}$ does not assume anything about the base classifier. We can therefore plug in our modified base classifier, $\textsc{Aggregate}_{i=1}^k f(\rvx_i)$ and similarly prove that the radius will hold with high probability.

\begin{corollary}[Smooth-Reduce certificates]
    \label{thm:Smooth-Max}
     Let $c_A \in [C]$, and $\underline{p_A}', \overline{p_B}' \in [0,1]$ be the probability estimates from the Smooth-Reduce classifier, $\bar{f}$. If 
     \begin{equation*}
        \sP_{\rvz} \left( \bar{f}(\rvx + \rvz) = c_A\right) \geq \underline{p_A}' \geq \overline{p_{B}}' \geq \max_{c \neq c_A} \sP_{\rvz} \left( \bar{f}(\rvx + \rvz) = c\right)\,,
    \end{equation*}
    then $\bar{f}(\rvx + \delta) = c_A$ for all $\delta$ satisfying $\| \delta \|_2 \leq R$, where 
    \begin{equation}
        R = \frac{\sigma_{\rvz}}{2} \left( \Phi^{-1}(\underline{p_A}') - \Phi^{-1}(\overline{p_B}') \right).
        \label{eq:rs_smoothreduce}
    \end{equation}
\end{corollary}

\begin{small}
\begin{algorithm}[!t]
    \caption{Smooth-Reduce Certification Algorithm}
    \label{alg:certification}
    \begin{algorithmic}
       \STATE \textit{\# certify the robustness of $\bar{f}   $ around $x$}
       \STATE \textbf{function} \textsc{Certify}($f$, $\sigma$, $x$, $n_0$, $n$, $\alpha$)
       \STATE \quad $\{x_i\}$ $\leftarrow$ \textsc{Patchify($x$, k)}
       \STATE \quad $\texttt{counts0} \leftarrow$\textsc{SmoothReduceUnderNoise}$(f, \{x_i\}, n_0, \sigma)$
       \STATE \quad $\hat{c}_A \leftarrow$ top index in \texttt{counts0}
       \STATE \quad $\texttt{counts} \leftarrow \textsc{SmoothReduceUnderNoise}(f, \{x_i\}, n, \sigma)$
       \STATE \quad $\underline{p_A} \leftarrow \textsc{LowerConfBound}$($\texttt{counts}[\hat{c}_A]$, $n$, $1 - \alpha$) 
       \STATE \quad \textbf{if} $\underline{p_A} > \frac{1}{2}$ \textbf{return} prediction $\hat{c}_A$ and radius $\sigma \, \Phi^{-1}(\underline{p_A})$
       \STATE \quad \textbf{else return} ABSTAIN
        \STATE      
        \STATE \textit{\# Sampling with Smooth-Reduce classifiers}  
       \STATE \textbf{function} \textsc{SmoothReduceUnderNoise}$(f,\,\{x_i\}, n, \sigma)$
       \STATE \quad \texttt{counts} $\leftarrow [0, 0, ... C\text{ times}]$
       \STATE \quad \textbf{for} $j=1:n$ 
       \STATE \qquad $\{z_{i,j}\} \rightarrow$ Sample from $\gN(0, \sigma^2\rmI)$
       \STATE \qquad $\{\hat{y}_{i,j}\} = \{f(x_i + z_{i, j})\}$
       \STATE \qquad $\hat{y}_j = \textsc{Reduce} (\{ \hat{y}_{i,j} \}) $~~\text{\# Reduce over patches}
       \STATE \qquad \texttt{counts}$[\argmax_{c \in C} \hat{y}_j]+= 1$  
       \STATE  \quad return \texttt{counts}
    \end{algorithmic}
\end{algorithm}
\end{small}

Following standard practice in evaluating RS algorithms, we modify the \textsc{Predict} and \textsc{Certify} subroutines as in \cite{cohen2019certified}. For the prediction step, we create $n$ copies of our input set. We then add independent noise for each element, and take the aggregate (maximum or average) for each class over all copies. The classifier either returns the most likely class over the aggregated logits over the copies, or abstains if the confidence of the probability estimate is low. We repeat the same modifications for the certification process to estimate lower bounds for the probability of the correct class. \Algref{alg:certification} shows a pseudo-code representation of the certification algorithm. We also show a diagrammatic representation of our approach in \Figref{fig:block}. For the Smooth-Max routine, we also scale the predicted logits using softmax over the classes for each copy. Similar to \cite{cohen2019certified}, our classifier abstains unless the event, $\underline{p_A'} \geq 1/2$ holds with probability larger than $1-\alpha$.

Notice that since we are estimating the lower bound on $\underline{p_A}$, the robustness guarantee holds in high probability. However, we can leverage the benefits of ensembling in each step to improve the success probability. For example, consider that there exists an adversarial example $\delta$ for any sub-classifier $f_i$ such that $\|\delta\|_2 \leq R$. Suppose the probability of such an event occurring can be upper bounded 
by $\alpha$. Then, the probability of $\delta$ to be an adversarial example for $\bar{f}$ is at most $\alpha/k$; see the appendix for a detailed discussion. Theoretically, this allows us to achieve the same performance as $\hat{f}$ with $k$ times fewer samples. However, in practice, this may lead to a high abstention rate if the base classifier $f$ is itself not robust enough to noise. 

We now show that Smooth-Max classifiers are inherently ``harder'' to attack than smooth classifiers. For this, we adapt a proof technique from Salman~\etal~\cite{salman2019provably} to get the following result.

\begin{theorem}
Let $\bar{f}$ and $\hat{f}$ be Smooth-Max and smooth classifiers as defined above. Let $R_{\bar{f}}$ and $R_{\hat{f}}$ be their corresponding certified radii derived using \Algref{alg:certification}. Assuming that the correct class is $c_A$ and the $|\gX|$ is high enough, then,
$R_{\bar{f}} \geq R_{\hat{f}}$.
\end{theorem}

The proof relies on the Smooth-Max algorithm selecting the patch with highest probability at every step. We then use the monotonicity of $\Phi^{-1}$ to finish the argument. We provide a detailed proof in the supplement. 

\noindent\textbf{Smooth-Max versus Smooth-Mean.}
 We now reflect upon the two flavors of Smooth-Reduce. By construction, if the base classifier succeeds on patches then Smooth-Max should intuitively perform at least as well than standard randomized smoothing. During inference, the Smooth-Max classifier picks the best of possible patches in the input set. Therefore, in expectation, we hope that the Smooth-Max classifier will be more robust. The Smooth-Mean classifier, on the other hand, improves predictions using averaging to reduce the variance. Intuitively, patches that are classified with low confidence are countered by patches with very high confidence. Our observation is that the Smooth-Mean classifier abstains less frequently as compared to Smooth-Max and base RS classifiers. This also showcases one of the limitations of Smooth-Max classifiers: if the base classifier, $f$ is very robust to noise, then a bad patch can consistently be chosen leading to the Smooth-Max classifier abstaining more often. In such a case, the Smooth-Mean classifier rectifies this by not relying on a single patch, leading to fewer abstentions. We also see evidence of this behavior in our results as discussed below. 

\section{Experiments and Results}
\label{sec:expts}

\subsection{Certificates for Image Classifiers}

We evaluate our approach by certifying classifiers trained on CIFAR-10 and ImageNet. To meaasure the performance, we consider three metrics: (1) The approximate certified accuracy with respect to the radius, (2)  Average Certified Radius (ACR), and (3) the abstention rate. We define average certified radius as in \cite{Zhai2020MACERAA}, where for each $(x_i,y_i)$ in the test set, $D_\text{test}$, and the corresponding certified radius, $R_i$, we calculate the ACR as $\frac{1}{|D_\text{test}|}\sum_{(x_i, y_i} \mathbf{1} [\bar{f}(x_i)=y_i] R_i$. The abstention rate is defined as the fraction of abstentions for the given test set. 
We show that our approach improves upon all the metrics over other randomized smoothing methods. 

\noindent\textbf{Setup:} In order to test our approach, we leverage the base classifiers trained in \cite{salman2019provably} for CIFAR-10 and ImageNet. The base classifiers have been adversarially trained to be robust to varying Gaussian distributions as well as smooth adversarial attacks. We use the models with the highest reported performance for each value of the variance, $\sigma$. Further, we use $n_0=100$ samples for prediction, and $n=100,000$ for certifying CIFAR-10, and similarly $n_0=100, n=100,000$ for Imagenet. We choose the best reported models in \cite{Zhai2020MACERAA,Alfarra2020DataDR} and \cite{Horvath2021BoostingRS} for comparisons with the same setting unless otherwise stated. We retrained models for MACER~\cite{Zhai2020MACERAA} for CIFAR-10 and Imagenet. However, for others, we used reported numbers as the code and models were not available at the time of publication. Additional details are available in the appendix and the accompanying supplementary material.

\begin{figure}[ht]
    \centering
    \includegraphics[width=0.99\linewidth]{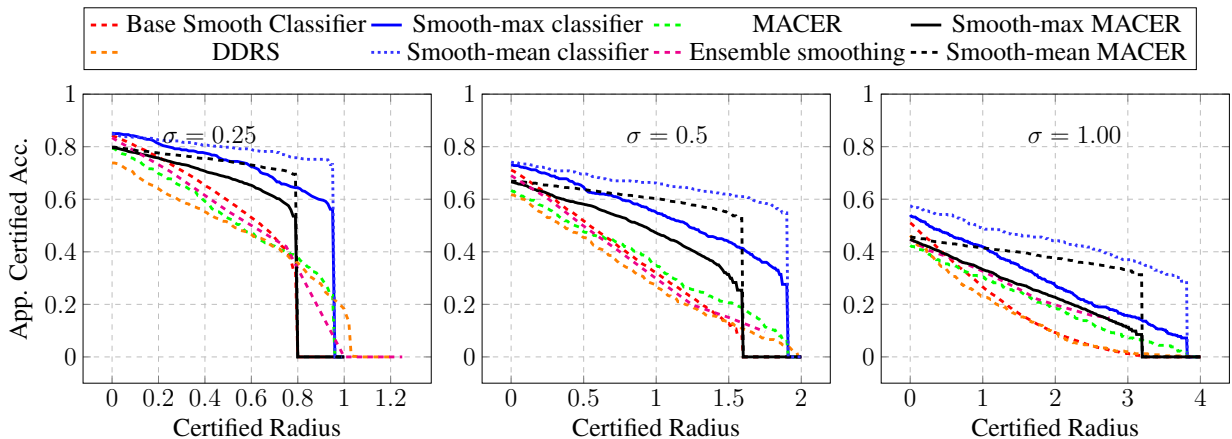}
    \caption{\textbf{Certified Accuracies for Cifar-10.}\sl \small Smooth-Max classifiers provide better certificates as compared to other approaches. For the same number of copies of the input set, we see that Smooth-Reduce outperforms both in terms of certified radius and also abstains less frequently than all other approaches. Additionally, Smooth-Reduce can be effortlessly integrated with other improvements in RS (for example, MACER~\cite{Zhai2020MACERAA}) without any retraining to achieve improved certificates.} 
    \label{fig:cifar10_comp}
\end{figure}

\noindent\textbf{Results on CIFAR-10:} For CIFAR-10, we use the pretrained Resnet-110 from \cite{salman2019provably}. Since the inputs are required to be $32\times32$ images, we resize each input image to be $36\times36$ and sample $25$ patches of the size $32\times32$ using either a random or uniform sampling process. We certify the first $500$ CIFAR-10 test images with both variants of our Smooth-Reduce algorithm for $\sigma={0.25, 0.5,1.0}$. Certificates for a larger sample of the dataset are also available in the Appendix. Note that we use the corresponding adversarially trained noise models from \cite{salman2019provably} as our base models.  \Figref{fig:cifar10_comp} shows results our experiments. We also compare Smooth-Reduce with standard randomized smoothing~\cite{cohen2019certified,salman2019provably}, DDRS~\cite{Alfarra2020DataDR}, MACER~\cite{Zhai2020MACERAA} and Ensemble smoothing~\cite{Horvath2021BoostingRS}. We demonstrate that the Smooth-Mean and Smooth-Max algorithms outperform all other approaches by a significant margin in terms of certified accuracy. Further, the average certified radius for Smooth-Mean and Smooth-Max exceed that of other approaches by at least 25\% and 14\% respectively. We also see that the improvements increase in magnitude as the noise variance increases. Finally, also note that, we are also successful in reducing abstention rates. \Tableref{t:cifar10} shows the average certified radii and abstention rates for the various certification algorithms. Also see that Smooth-Mean classifiers tend to abstain far less often as compared to Smooth-Max classifiers.

We also study the effect of confidence calibration by ensuring both SmoothAdv and Smooth-Reduce classifiers see the same number of overall patches. To ensure fair comparison, we run SmoothAdv certification with $N=100k$ and $\alpha=0.001$ and Smooth-Reduce with $N=10,000, k=10,$ and use $\alpha$ to be $0.01$ for each sub-classifier. \Figref{fig:conf_calibration}(b) shows that under the same number of samples, Smooth-Reduce is able to certify larger radii while improving the certified accuracy.

\begin{small}
\begin{table}[ht]
        \centering
        \caption{\textbf{Results for CIFAR-10.}Average Certified Radii and Abstention rates for CIFAR-10. The best performer is bolded and the second best is italicized. Results with a $*$ are from reported results in the publication.}
        \label{t:cifar10}
        \begin{tabular}{c c c c c c c}
        \toprule[1.5pt]
        Algorithm & \multicolumn{3}{c}{ACR $\uparrow$} & \multicolumn{3}{c}{Abst. Rate $\downarrow$} \\
        \midrule
        $\sigma$ & 0.25  & 0.5 & 1.0 & 0.25  & 0.5 & 1.0 \\
        \midrule 
         Smooth-Max (Ours)  & 0.713  & 1.269 & 2.04 &  0.017 &  0.047 & 0.1064  \\
         Smooth-Mean (Ours) & \textbf{0.759} & \textit{1.426} & \textit{2.665} & \textbf{0.005} & \textit{0.046} & \textit{0.030} \\
         Smooth-Max Macer (Ours) & 0.701 & 1.209 & 1.910 & 0.018 & 0.058 & 0.110 \\
         Smooth-Mean Macer (Ours) & \textit{0.754} & \textbf{1.462} & \textbf{2.736} & 0.005 & \textbf{0.0136} & \textbf{0.020} \\
         Standard RS~\cite{cohen2019certified,salman2019provably} & 0.605 & 0.899 & 1.185 & 0.039 & 0.1099 & 0.2582 \\
         MACER~\cite{Zhai2020MACERAA} & 0.517 & 0.682 & 0.767 & 0.206 & 0.366 & 0.576 \\
         DDRS~\cite{Alfarra2020DataDR} & 0.678 & 0.942 & 1.185 & 0.048 & 0.122 & 0.244\\
         Ensemble RS$^*$~\cite{Horvath2021BoostingRS} & 0.583 & 0.756 & 0.788 & - & -  & -  \\
        \bottomrule 
        \end{tabular}
\end{table}
\end{small}
\begin{small}
\begin{table}[ht]
        \centering
        \caption{\textbf{Results for Imagenet.}Average Certified Radius and Abstention rates for ImageNet. The best performer is bolded and the second best is italicized. Results with a $*$ are reported results in the publication. }
        \begin{tabular}{c c c c c c c}
        \toprule[1.5pt]
        Algorithm & \multicolumn{3}{c}{ACR $\uparrow$} & \multicolumn{3}{c}{Abst. Rate $\downarrow$} \\
        \midrule
        $\sigma$ & 0.25  & 0.5 & 1.0 & 0.25  & 0.5 & 1.0 \\
        \midrule 
         Smooth-Max (Ours)  & \textit{0.767} & \textit{1.453} & \textit{2.611} & \textit{0.008} & \textit{0.038}  & \textit{0.108}  \\
         Smooth-Mean (Ours) & \textbf{0.786} & \textbf{1.513} & \textbf{2.931} & \textbf{0.002} & \textbf{0.024} & \textbf{0.048}\\
         Standard RS~\cite{cohen2019certified,salman2019provably} &  0.729 & 1.327 & 2.204 & 0.02 & 0.098 & 0.22\\
         MACER$^*$~\cite{Zhai2020MACERAA} & 0.544 & 0.831 & 1.008 & - & - & - \\
         Ensemble RS$^*$~\cite{Horvath2021BoostingRS} & 0.545 & 0.868 & 1.108 & - & -& -  \\
        \bottomrule 
        \end{tabular}
\end{table}
\end{small}
\noindent\textbf{Results on Imagenet:} We also test our approach on $500$ randomly chosen images from Imagenet to certify a pretrained Resnet-20 model from \cite{salman2019provably}. We resize our inputs to $256\times256$ and sample $224\times224$ sub-patches. We use $4, 8,$ and $16$ patches for our approach with $n_0=100$ and $n=100k$ for certification.  Smooth-Reduce improves upon standard randomized smoothing as shown in \Figref{fig:imagenet_comps}. Specifically, Smooth-Mean performs the best, having 32.9\% relatively higher average certified radius as compared to standard randomized smoothing. Smooth-Max performs the second-best. Also notice that Smooth-Max performance remains stable with the number of patches.

\begin{figure}[ht]
    \centering
    \includegraphics[width=0.99\linewidth]{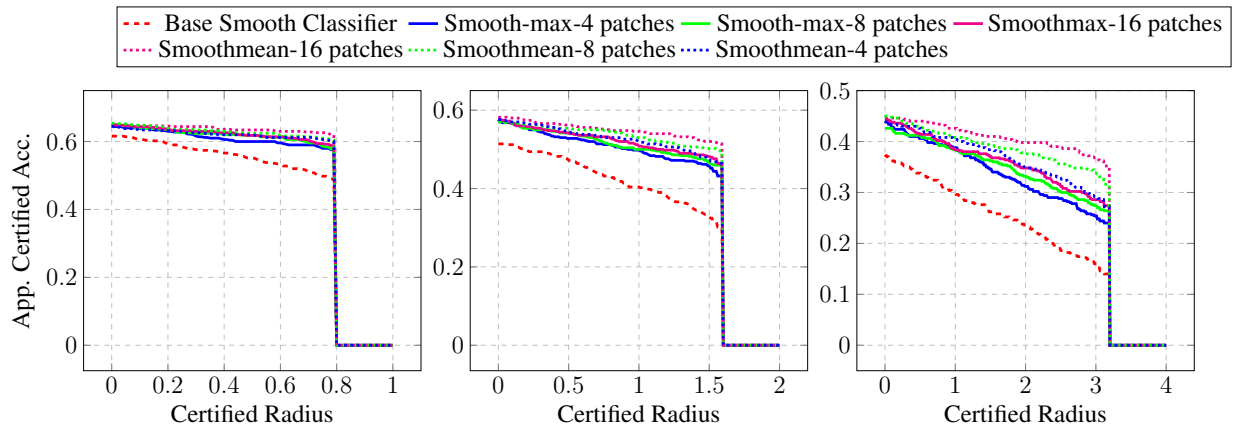}
\caption{\textbf{Certification for ImageNet.}\sl \small We see similar performance improvements of Smooth-Reduce over standard RS. Increasing number of patches does not affect Smooth-Max certificates significantly. However, Smooth-Mean classifiers have a higher approximate certified accuracy as the number of patches increase, especially as the noise variance increases. More detailed results can be found in the appendix.}
\label{fig:imagenet_comps}
\end{figure}

\subsection{Certificates for Video Classifiers}
\label{subsec:video_clf}

Video classifiers often employ aggregation over chunks from the video stream to tackle the problem of varying number of frames~\cite{crasto2019mars}(see \Figref{fig:video_block} in appendix). Therefore, we propose Smooth-Reduce as a natural method to certify such classifiers. While standard RS certificates have not been reported for such models, we observe in \Figref{fig:vidresults}(a) and \Figref{appdxfig:vid_results}(see Appendix) that certified accuracies using RS are still low. As a remedy, we adapt our Smooth-Reduce algorithm to videos. 

While the natural approach here would be to simply look at overlapping chunks as analogues for patches, initial tests showed catastrophic loss of accuracy when we use single chunks for prediction. We therefore sample  overlapping sub-videos with $t$ frames instead. Each sub-video consists of a fixed number of chunks; each with $m$ frames. The base video classifier aggregates over these chunks to produce a prediction. We  repeat the same process of smoothing and aggregation over the sub-videos instead of chunks, and label this as Smooth-Reduce-$(t,m)$ where $t$ is the number of frames in each sub-video and $m$ is the number of frames in each chunk.  \Figref{fig:video_block} in the appendix shows a pictorial representation;see \Secref{appdx:sec_vid_classifiers} for a more detailed description. Note here that the base classifier itself is an ensemble over multiple $16$ frame chunks.

\noindent\textbf{Experiments and Results.} We test our approach on 3D ResNeXt-101 RGB \cite{xie2017aggregated} trained on UCF-101~\cite{soomro2012ucf101}.
We retrain models initialized with weights from \cite{crasto2019mars} using clips of 16 consecutive RGB frames with Gaussian noise augmentation. 
Similar to the setting of \cite{crasto2019mars}, we use SGD with weight decay of $0.0005$, momentum of $0.9$, and initial learning rate of $0.1$. 
We used the first train split and the first test split for training and testing our model, respectively. Additional training details can be found in the appendix.

For inference, the video classifier \cite{crasto2019mars} follows these steps: (1) the input video stream is split into non overlapping chunks of $16$ frames each, (2) the model predictions on these chunks are averaged, and returned as the output class.
We run Smooth-Reduce certification by first sampling $64$ frame or $128$ frame sub-videos for a video stream (analogous to patching for images) to create the input set. Then we plug in the video classifier inference routine to predict classes for noisy copies of each sub-video. See \Figref{fig:vidresults} for results. Note that Smooth-Max and Smooth-Mean both outperform the standard randomized smoothing classifier. 

\noindent\textbf{Limitations of Randomized Smoothing for Videos.} We encountered several challenges while attempting to certify video classifiers. A significant challenge was training noise robust classifiers. We observe that adding Gaussian noise to video data often led to catastrophic decreases in accuracy. This could be an artifact of the architecture which averages predictions over frames by itself. Further, the memory requirements often became insurmountable to get high probability certificates. Our certificates here have been estimated using $n_0=10$ samples for prediction, and $n=1000$ samples for certification with a failure probability of $\alpha=0.001\%$. However, Smooth-Reduce allows for lower sample complexity (see appendix), allowing for Smooth-Mean models to still achieve non-trivial certified accuracies.

\begin{figure}[ht]
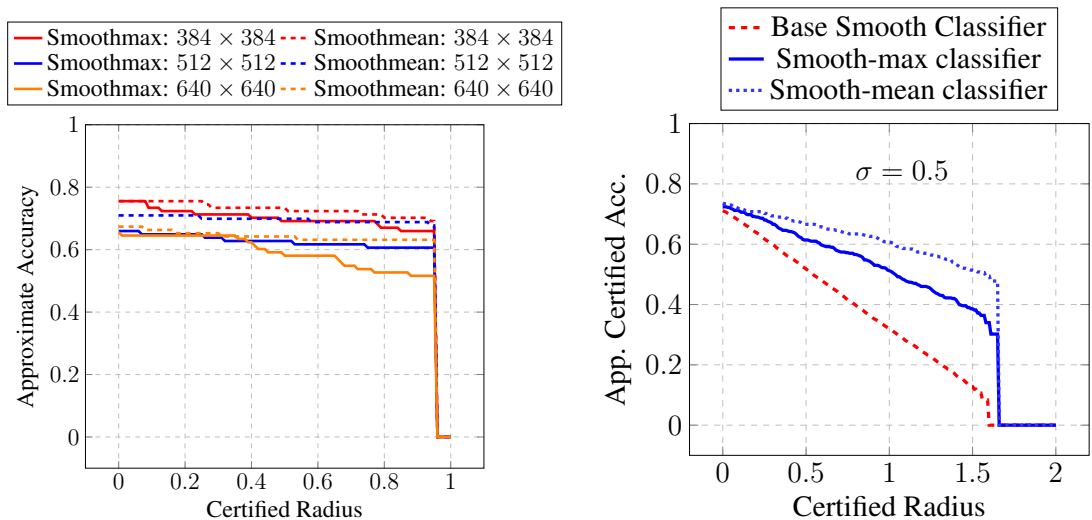

\centering
\begin{tabular}{c c}
    \includegraphics[width=0.45\linewidth]{Plots/imagenet_vary_025.tex} &
   \includegraphics[width=0.4\linewidth]{Plots/cifar10_same_samples_05.tex}
\end{tabular}
    \caption{\textbf{Effect of resizing.(L)} \sl Notice that as the input size increases, the information in each patch correspondingly decreases. The base classifier performs worse overall for each patch leading to lower certified accuracy. (b) \textbf{Certification with Confidence Calibration (R)}. Under the same number of overall samples and calibrated failure probabilites, Smooth-Reduce out-performs standard RS~\cite{salman2019provably} in both certification radius and certified accuracy.}
    \label{fig:varying}
    \label{fig:conf_calibration}
\end{figure}

\subsection{Ablation Studies.} We further analyse the various components of our approach.

\noindent\textbf{Number of patches:} We measure the effect of the number of samples used for Smooth-Reduce certification on Imagenet classifiers.
\Figref{fig:imagenet_comps} shows that Smooth-Max classifiers are relatively unaffected by the number of samples chosen. However, we see that increasing number of patches improves performance of Smooth-Mean certificates. This can be attributed to better empirical estimates as the number of samples increase, and also verifies our theoretical analysis; see appendix.
Note too that this difference is more evident for higher noise variances. 

\noindent\textbf{Effect of Resizing:} Another component of Smooth-Reduce is the resizing step undertaken while sampling. While theoretically it should not affect the radius, the base classifier does assume that the features would be of a certain size. We analyse the effect of the resizing step by resizing Imagenet test images to $384\times384, 512\times512$ and, $640\times640$ and sampling $16$ patches of $224\times224$ randomly. In \Figref{fig:varying}(a), we observe that as resizing becomes more extreme, the certified accuracy falls in tandem with base accuracy.

\noindent\textbf{Random v/s dense sampling:} Since sampling of patches plays a large role in creating a diverse input set, we also analyse the effect of two sampling approaches; dense and uniform random. For random sampling, we select patches randomly with replacement from the resized input image, and discarding any `invalid' patches that fall outside the image borders. For dense sampling, we sample overlapping patches with a specified stride length. We evaluate if the sampling approach affects the certificates by sampling $25$ patches for each method. We observe that the sampling process does not affect the certification process as long as the number of patches are high enough~(see \Figref{appdxfig:cifar10_random_dense} in appendix). 

\noindent\textbf{Effect of subvideo/chunk sizes:} We analyze the effect using different subvideo and chunk sizes on UCF101 certifcation. Specifically, we use either 64 or 128 frames for the size of subvideos, and 8 or 16 frames for the size of chunks. \Figref{fig:vidresults} suggest that we would benefit from having a larger subvideo or chunk sizes. In fact, we observe that using 128 frame subvideos and 64 frame chunks yield the highest certification accuracy.

\begin{figure}[!t]
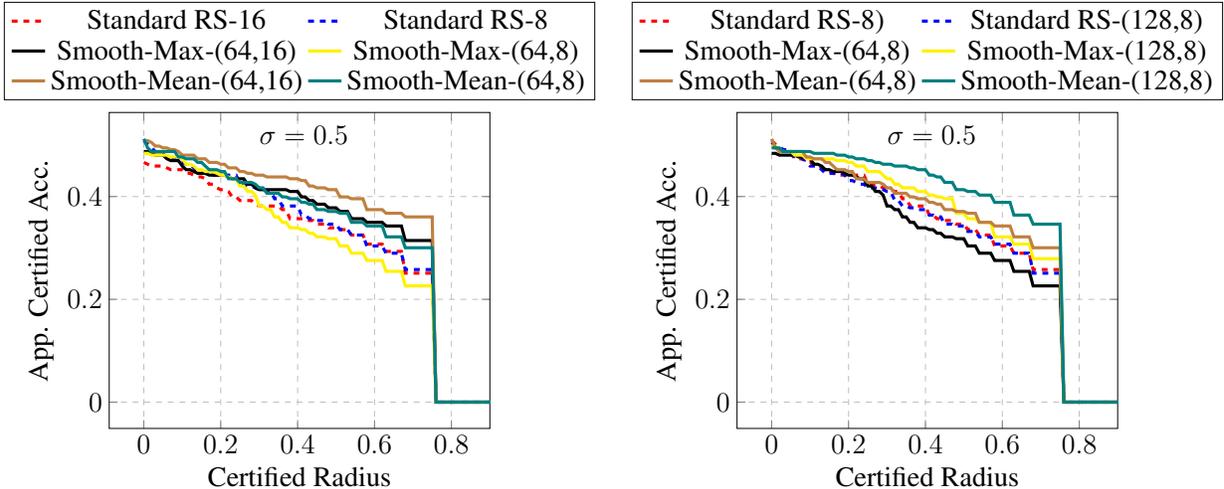

\centering
    \begin{tabular}{c c}
     \includegraphics[width=0.48\linewidth]{Plots/ucf101_ablation_chunk_size.tex} &
     \includegraphics[width=0.48\linewidth]{Plots/ucf101_ablation_subvideo_size.tex}
    \end{tabular}
    \caption{\textbf{Results on UCF-101.} \sl \small For standard RS classifiers, the number represents the number of frames in each chunk. For Smooth-Reduce classifiers, the first number is the subvideo size, and the second is the number of frames in each chunk.  (\textbf{Left}) \textit{Varying chunk sizes.}Our modified Smooth-Reduce classifiers provide higher certified accuracies as compared to standard RS on UCF-101 videos. Also observe that for a fixed subvideo size, RS and Smooth-Reduce classifiers using larger chunks (16 frames over 8 frames) are more robust. (\textbf{Right}) \textit{Varying subvideo sizes.} Smooth-Reduce classifiers using larger subvideo sizes are more robust. The difference is more pronounced for Smooth-Mean classifiers which outperform standard RS by a large margin. More results can be found in the Appendix.}
    \label{fig:vidresults}
\end{figure}

\section{Discussion and Conclusions}
\label{sec:discussion}

We present Smooth-Reduce, an extension of the randomized smoothing approach proposed in \cite{cohen2019certified}. We empirically and theoretically proved that Smooth-Reduce classifiers improve over standard randomized smoothing in terms of certified radii, as well as abstention rate. Our approach relies on the performance boosting properties of ensemble classifiers, which we emulate by creating an input set using patches. We show better certification performance as compared to other smoothing methods under two different aggregation schemes. A major benefit is that our approach requires no additional re-training. 

Our approach also does not make any assumptions on the base classifier. Therefore, Smooth-Reduce can plugged in effortlessly into other certified classifiers, such as MACER~\cite{Zhai2020MACERAA}, and achieve improved certificates. Some limitations of our approach persist. Firstly, we require higher inference-time computation than standard RS approaches. However, note that in comparison to other ensembling approaches~\cite{Horvath2021BoostingRS,Yang2021OnTC}, our method does not require training multiple classifiers. Further, we have not studied adaptive attacks for this scheme, and constructing reasonable attacks for such classifiers (and verifying these certificates empirically) is a complex research question in and of itself. We leave these directions to future work.

\section*{Acknowledgements}

Part of this work was performed when AJ was a summer intern at Bosch AI, where it was supported by DARPA grant HR11002020006. AJ, MC, MP, and CH would also like to acknowledge NSF grants CCF-2005804 and CCF-1815101,  USDA/NIFA grant 2021-67021-35329, and ARPA-E DIFFERENTIATE grant DE-AR0001215.

\bibliographystyle{splncs04}
\bibliography{egbib}

\begin{thebibliography}{10}
\providecommand{\url}[1]{\texttt{#1}}
\providecommand{\urlprefix}{URL }
\providecommand{\doi}[1]{https://doi.org/#1}

\bibitem{addepalli2021boosting}
Addepalli, S., Jain, S., Sriramanan, G., Babu, R.V.: Boosting adversarial
  robustness using feature level stochastic smoothing. In: Proceedings of the
  IEEE/CVF Conference on Computer Vision and Pattern Recognition. pp. 93--102
  (2021)

\bibitem{Alfarra2020DataDR}
Alfarra, M., Bibi, A., Torr, P.H.S., Ghanem, B.: Data dependent randomized
  smoothing. ArXiv  \textbf{abs/2012.04351} (2020)

\bibitem{Athalye2018obfuscated}
Athalye, A., Carlini, N., Wagner, D.: Obfuscated gradients give a false sense
  of security: Circumventing defenses to adversarial examples. In: ICML (2018)

\bibitem{carlini2017groundtruth}
Carlini, N., Katz, G., Barrett, C., Dill, D.L.: Ground-truth adversarial
  examples. arXiv  (2017)

\bibitem{Carlini2017cwl2}
Carlini, N., Wagner, D.: Towards evaluating the robustness of neural networks.
  IEEE (SP)  (2017)

\bibitem{cohen2019certified}
Cohen, J., Rosenfeld, E., Kolter, Z.: Certified adversarial robustness via
  randomized smoothing. In: ICML. PMLR (2019)

\bibitem{crasto2019mars}
Crasto, N., Weinzaepfel, P., Alahari, K., Schmid, C.: Mars: Motion-augmented
  rgb stream for action recognition. In: Proceedings of the IEEE/CVF Conference
  on Computer Vision and Pattern Recognition. pp. 7882--7891 (2019)

\bibitem{Dosovitskiy2021AnII}
Dosovitskiy, A., Beyer, L., et~al.: An image is worth 16x16 words: Transformers
  for image recognition at scale. In: ICLR (2020)

\bibitem{Goodfellow2018existence}
Goodfellow, I.: Defense against the dark arts: An overview of adversarial
  example security research and future research directions. arxiv preprint
  \textbf{1806.04169} (2018)

\bibitem{Goodfellow-et-al-2016}
Goodfellow, I., Bengio, Y., Courville, A.: Deep Learning. MIT Press (2016),
  \url{http://www.deeplearningbook.org}

\bibitem{Horvath2021BoostingRS}
Horv'ath, M.Z., M{\"u}ller, M.N., Fischer, M., Vechev, M.T.: Boosting
  randomized smoothing with variance reduced classifiers. ArXiv
  \textbf{abs/2106.06946} (2021)

\bibitem{huang2017safety}
Huang, X., Kwiatkowska, M., Wang, S., Wu, M.: Safety verification of deep
  neural networks. Computer Aided Verification (CAV)  (2017)

\bibitem{jeong2021smoothmix}
Jeong, J., Park, S., Kim, M., Lee, H.C., Kim, D.G., Shin, J.: Smoothmix:
  Training confidence-calibrated smoothed classifiers for certified robustness.
  Advances in Neural Information Processing Systems  \textbf{34} (2021)

\bibitem{katz2017reluplex}
Katz, G., Barrett, C., Dill, D., Julian, K., Kochenderfer, M.: Reluplex: An
  efficient smt solver for verifying deep neural networks. arXiv preprint
  arXiv:1702.01135  (2017)

\bibitem{katz2017towards}
Katz, G., Barrett, C., Dill, D.L., Julian, K., Kochenderfer, M.J.: Towards
  proving the adversarial robustness of deep neural networks. arXiv preprint
  (2017)

\bibitem{lecuyer2019certified}
Lecuyer, M., Atlidakis, V., Geambasu, R., Hsu, D., Jana, S.: Certified
  robustness to adversarial examples with differential privacy. In: 2019 IEEE
  Symposium on Security and Privacy (SP). pp. 656--672. IEEE (2019)

\bibitem{pmlr-v108-levine21a}
Levine, A.J., Feizi, S.: Improved, deterministic smoothing for {$L_1$}
  certified robustness. In: ICML (2021)

\bibitem{Liu2020EnhancingCR}
Liu, C., Feng, Y., Wang, R., Dong, B.: Enhancing certified robustness of
  smoothed classifiers via weighted model ensembling. ArXiv
  \textbf{abs/2005.09363} (2020)

\bibitem{madry2018towards}
Madry, A., Makelov, A., Schmidt, L., Tsipras, D., Vladu, A.: Towards deep
  learning models resistant to adversarial attacks. In: ICLR (2018),
  \url{https://openreview.net/forum?id=rJzIBfZAb}

\bibitem{Raghunathan2018CertifiedDA}
Raghunathan, A., Steinhardt, J., Liang, P.: Certified defenses against
  adversarial examples. In: ICLR (2018)

\bibitem{Raghunathan2018semidefinite}
Raghunathan, A., Steinhardt, J., Liang, P.S.: Semidefinite relaxations for
  certifying robustness to adversarial examples. In: NeurIPS (2018)

\bibitem{salman2019provably}
Salman, H., Yang, G., Li, J., Zhang, P., Zhang, H., Razenshteyn, I., Bubeck,
  S.: Provably robust deep learning via adversarially trained smoothed
  classifiers. In: NeurIPS (2019)

\bibitem{Samangouei2018DefenseGANPC}
Samangouei, P., Kabkab, M., Chellappa, R.: Defense-gan: Protecting classifiers
  against adversarial attacks using generative models. ArXiv
  \textbf{abs/1805.06605} (2018)

\bibitem{soomro2012ucf101}
Soomro, K., Zamir, A.R., Shah, M.: Ucf101: A dataset of 101 human actions
  classes from videos in the wild (2012)

\bibitem{sukenik2021intriguing}
S{\'u}ken{\'\i}k, P., Kuvshinov, A., G{\"u}nnemann, S.: Intriguing properties
  of input-dependent randomized smoothing. arXiv preprint arXiv:2110.05365
  (2021)

\bibitem{szegedy2013intriguing}
Szegedy, C., Zaremba, W., Sutskever, I., Bruna, J., Erhan, D., Goodfellow, I.,
  Fergus, R.: Intriguing properties of neural networks. International
  Conference on Learning Representations  (2014)

\bibitem{teng2019ell_1}
Teng, J., Lee, G.H., Yuan, Y.: $\ell_1$ adversarial robustness certificates: a
  randomized smoothing approach. OpenReview  (2019),
  \url{https://openreview.net/forum?id=H1lQIgrFDS}

\bibitem{tjeng2017evaluating}
Tjeng, V., Xiao, K., Tedrake, R.: Evaluating robustness of neural networks with
  mixed integer programming. arXiv preprint arXiv:1711.07356  (2017)

\bibitem{Trockman2022PatchesAA}
Trockman, A., Kolter, J.Z.: Patches are all you need? ArXiv
  \textbf{abs/2201.09792} (2022)

\bibitem{vershynin2018high}
Vershynin, R.: High-dimensional probability: An introduction with applications
  in data science, vol.~47. Cambridge university press (2018)

\bibitem{mart}
Wang, Y., Zou, D., Yi, J., Bailey, J., Ma, X., Gu, Q.: Improving adversarial
  robustness requires revisiting misclassified examples. In: ICLR (2019)

\bibitem{Wang2021ImprovingAR}
Wang, Y.: Improving adversarial robustness for free with snapshot ensemble.
  ArXiv  \textbf{abs/2110.03124} (2021)

\bibitem{wang2019improving}
Wang, Y., Zou, D., Yi, J., Bailey, J., Ma, X., Gu, Q.: Improving adversarial
  robustness requires revisiting misclassified examples. In: ICLR (2019)

\bibitem{Weng2018TowardsFC}
Weng, T.W., Zhang, H., Chen, H., Song, Z., Hsieh, C.J., Boning, D.S., Dhillon,
  I.S., Daniel, L.: Towards fast computation of certified robustness for relu
  networks. In: ICML (2018)

\bibitem{wong2018provable}
Wong, E., Kolter, Z.: Provable defenses against adversarial examples via the
  convex outer adversarial polytope. In: ICML. PMLR (2018)

\bibitem{xie2017aggregated}
Xie, S., Girshick, R., Dollár, P., Tu, Z., He, K.: Aggregated residual
  transformations for deep neural networks (2017)

\bibitem{Yang2020RandomizedSO}
Yang, G., Duan, T., Hu, E.J., Salman, H., Razenshteyn, I.P., Li, J.: Randomized
  smoothing of all shapes and sizes. In: ICML (2020)

\bibitem{Yang2021OnTC}
Yang, Z., Li, L., Xu, X., Kailkhura, B., Xie, T., Li, B.: On the certified
  robustness for ensemble models and beyond. ArXiv  \textbf{abs/2107.10873}
  (2021)

\bibitem{yin2020defense}
Yin, Z.and~Wang, H., Wang, J., Tang, J., Wang, W.: Defense against adversarial
  attacks by low-level image transformations. International Journal of
  Intelligent Systems  (2020)

\bibitem{Zhai2020MACERAA}
Zhai, R., Dan, C., He, D., Zhang, H., Gong, B., Ravikumar, P., Hsieh, C.J.,
  Wang, L.: Macer: Attack-free and scalable robust training via maximizing
  certified radius. ArXiv  \textbf{abs/2001.02378} (2020)

\bibitem{trades}
Zhang, H., Yu, Y., Jiao, J., Xing, E., El~Ghaoui, L., Jordan, M.: Theoretically
  principled trade-off between robustness and accuracy. In: ICML. pp.
  7472--7482 (2019)

\bibitem{Zhang2018mixupBE}
Zhang, H., Ciss{\'e}, M., Dauphin, Y., Lopez-Paz, D.: mixup: Beyond empirical
  risk minimization. ArXiv  \textbf{abs/1710.09412} (2018)

\end{thebibliography}

\clearpage
\appendix

\section{Results and Discussion}
\label{apdx:sec_add_results}

We support our observations in \Secref{sec:expts} with some additional results presented here.

\subsection{Additional Results for CIFAR-10}
\begin{table}[!htbp]
    \centering
    \caption{\textbf{Detailed results on CIFAR-10.}}
    \resizebox{\textwidth}{!}{
    \begin{tabular}{c |c| c c c c c c c c c c c c c c c c c}
        \toprule[1.5pt]
       $\sigma$ & Approach & \multicolumn{15}{c}{Radii} \\
       \midrule[1.5pt]
       {} & {} & 0.0 & 0.25 & 0.5 & 0.75 & 1.0 & 1.25 & 1.5 & 1.75 & 2.0 & 2.25 & 2.5 & 2.75 & 3.0 & 3.25 & 3.5 & 3.75 & 4.0  \\
       \midrule
       \multirow{7}{*}{0.25} & Smooth-Max (Ours) &  85 & 79 & 74 & 65 & 0 & 0 & 0 & 0 & 0 & 0 & 0 & 0 & 0 & 0 & 0 & 0 & 0
          \\  
        {} &  Smooth-Mean (Ours) & 84 & 82 & 79 & 76 & 0 & 0 & 0 & 0 & 0 & 0 & 0 & 0 & 0 & 0 & 0 & 0 & 0 \\
        {} & SmoothAdv~\cite{salman2019provably}  & 84 & 73 & 58 & 39 & 0 & 0 & 0 & 0 & 0 & 0 & 0 & 0 & 0 & 0 & 0 & 0 & 0 \\
        {} & MACER-Smooth-Max (Ours) & 80 & 75 & 68 & 60 & 0 & 0 & 0 & 0 & 0 & 0 & 0 & 0 & 0 & 0 & 0 & 0 & 0\\
        {} & MACER-Smooth-Mean (Ours) & 80 & 78 & 75 & 71 & 0 & 0 & 0 & 0 & 0 & 0 & 0 & 0 & 0 & 0 & 0 & 0 & 0 \\
        {} & MACER~\cite{Zhai2020MACERAA} & 79 & 67 & 52 & 40 & 0 & 0 & 0 & 0 & 0 & 0 & 0 & 0 & 0 & 0 & 0 & 0 & 0
 \\
        {} & DDRS~\cite{Alfarra2020DataDR}  & 73 & 61 & 51 & 39 & 18 & 0 & 0 & 0 & 0 & 0 & 0 & 0 & 0 & 0 & 0 & 0 & 0 \\
        {} & Ensemble~\cite{Horvath2021BoostingRS} & 83 & 70 & 55 & 42 & 0 & 0 & 0 & 0 & 0 & 0 &  0 & 0 & 0 & 0 & 0 & 0 & 0\\
        \midrule
        \multirow{7}{*}{0.50} & Smooth-Max (Ours) &  73 & 69 & 64 & 60 & 55 & 49 & 43 & 36 & 0 & 0 & 0 & 0 & 0 & 0 & 0 & 0 & 0 \\  
        {} &  Smooth-Mean (Ours) & 74 & 72 & 69 & 67 & 66 & 64 & 61 & 58 & 0 & 0 & 0 & 0 & 0 & 0 & 0 & 0 & 0\\
        {} & SmoothAdv~\cite{salman2019provably} & 72 & 61 & 50 & 40 & 31 & 20 & 10 & 0 & 0 & 0 & 0 & 0 & 0 & 0 & 0 & 0 & 0\\
        {} & MACER-Smooth-Max (Ours) & 67 & 62 & 59 & 54 & 48 & 41 & 35 & 26 & 0 & 0 & 0 & 0 & 0 & 0 & 0 & 0 & 0\\
        {} & MACER-Smooth-Mean (Ours) & 67 & 66 & 64 & 62 & 60 & 58 & 56 & 53 & 0 & 0 & 0 & 0 & 0 & 0 & 0 & 0 & 0 \\
        {} & MACER~\cite{Zhai2020MACERAA} & 63 & 56 & 47 & 43 & 34 & 25 & 20 & 11 & 0 & 0 & 0 & 0 & 0 & 0 & 0 & 0 & 0 \\
        {} & DDRS~\cite{Alfarra2020DataDR}  & 61 & 53 & 45 & 35 & 27 & 19 & 13 & 7 & 0 & 0 & 0 & 0 & 0 & 0 & 0 & 0 & 0 \\
        {} & Ensemble~\cite{Horvath2021BoostingRS} & 65 & 59 & 49 & 45 &  38 & 32 & 26 & 19 & 0 & 0 & 0 & 0 & 0 & 0 & 0 & 0 & 0 \\
        \midrule
        \multirow{7}{*}{1.0} & Smooth-Max (Ours) &  54 & 50 & 47 & 44 & 41 & 37 & 34 & 31 & 27 & 23 & 20 & 18 & 15 & 13 & 10 & 7 & 0 \\  
        {} &  Smooth-Mean (Ours) &  57 & 55 & 54 & 52 & 49 & 48 & 46 & 45 & 44 & 42 & 41 & 39 & 37 & 34 & 32 & 28 & 0\\
        {} & SmoothAdv~\cite{salman2019provably} & 50 & 44 & 36 & 28 & 21 & 18 & 13 & 9 & 7 & 5 & 4 & 2 & 1 & 0 & 0 & 0 & 0\\
        {} & MACER-Smooth-Max (Ours) & 45 & 42 & 39 & 35 & 33 & 30 & 27 & 23 & 22 & 19 & 16 & 13 & 11 & 9 & 7 & 3 & 0\\
        {} & MACER-Smooth-Mean (Ours) & 44 & 44 & 42 & 41 & 40 & 40 & 39 & 39 & 38 & 36 & 36 & 35 & 35 & 33 & 31 & 28 & 0 \\
        {} & MACER~\cite{Zhai2020MACERAA} & 42 & 39 & 35 & 32 & 30 & 27 & 25 & 20 & 18 & 15 & 12 & 9 & 7 & 6 & 4 & 1 & 0 \\
        {} & DDRS~\cite{Alfarra2020DataDR}  & 46 & 39 & 33 & 27 & 22 & 19 & 15 & 12 & 9 & 5 & 4 & 3 & 2 & 1 & 0 & 0 & 0 \\
        {} & Ensemble~\cite{Alfarra2020DataDR} & 49 & 43 & 37 & 30 & 23 & 18 & 16 & 13 & 11 & 9 & 5 & 0 & 0 & 0 & 0 & 0 & 0\\
        \bottomrule
    \end{tabular}}
    
    \label{appdxt:add_cifar10_results}
\end{table}

\subsection{Additional Results for ImageNet}
\label{appdx:imagenet_results}

\begin{table}[!htbp]
    \centering
    \caption{\textbf{Detailed results on ImageNet.}}
    \begin{tabular}{c |c| c c c c c c c c c}
        \toprule[1.5pt]
       $\sigma$ & Approach & \multicolumn{9}{c}{Radii} \\
       \midrule[1.5pt]
       {} & {} & 0.0 & 0.5 & 1.0 & 1.5 & 2.0 & 2.5 & 3.0 & 3.5 & 4.0 \\
       \midrule
       \multirow{9}{*}{0.25} & Smooth-Max-16(PGD) &  65 & 61 & 0 & 0 & 0 & 0 & 0 & 0 & 0 \\  
        {} &  Smooth-Mean-16 (PGD) & 64 & 63 & 0 & 0 & 0 & 0 & 0 & 0 & 0\\
        {} & Smooth-Max-16 (DDN) & 72 & 67 & 0 & 0 & 0 & 0 & 0 & 0 & 0 \\
        {} & Smooth-Mean-16 (DDN) &  72 & 69 & 0 & 0 & 0 & 0 & 0 & 0 & 0\\
        & Smooth-Max-8(PGD) &  65 & 61 & 0 & 0 & 0 & 0 & 0 & 0 & 0 \\  
        {} &  Smooth-Mean-8 (PGD) & 65 & 62 & 0 & 0 & 0 & 0 & 0 & 0 & 0\\
        {} & Smooth-Max-8 (DDN) & 72 & 67 & 0 & 0 & 0 & 0 & 0 & 0 & 0\\
        {} & Smooth-Mean-8 (DDN) & 72 & 68 & 0 & 0 & 0 & 0 & 0 & 0 & 0\\
        {} & SmoothAdv~\cite{salman2019provably} & 61 & 55 & 0 & 0 & 0 & 0 & 0 & 0 & 0\\
        \midrule
        \multirow{9}{*}{0.50} & Smooth-Max-16(PGD) &  57 & 54 & 50 & 48 & 0 & 0 & 0 & 0 & 0 \\  
        {} &  Smooth-Mean-16 (PGD) & 58 & 56 & 54 & 52 & 0 & 0 & 0 & 0 & 0\\
        {} & Smooth-Max-16 (DDN)  & 66 & 61 & 56 & 51 & 0 & 0 & 0 & 0 & 0\\
        {} & Smooth-Mean-16 (DDN) & 66 & 64 & 61 & 59 & 0 & 0 & 0 & 0 & 0 \\
        {} & Smooth-Max-8 (PGD) & 57 & 54 & 50 & 46 & 0 & 0 & 0 & 0 & 0 \\
        {} & Smooth-Mean-8 (PGD) & 56 & 55 & 53 & 50 & 0 & 0 & 0 & 0 & 0 \\
        {} & Smooth-Max-8 (DDN) & 66 & 61 & 56 & 48 & 0 & 0 & 0 & 0 & 0 \\
        {} & Smooth-Mean-8 (DDN) & 65 & 63 & 59 & 54 & 0 & 0 & 0 & 0 & 0 \\
        {} & SmoothAdv~\cite{salman2019provably} & 51 & 47 & 40 & 32 & 0 & 0 & 0 & 0 & 0\\
        \midrule
        \multirow{9}{*}{1.0} & Smooth-Max-16(PGD) &  44 & 41 & 38 & 37 & 34 & 31 & 28 & 0 & 0 \\  
        {} &  Smooth-Mean-16 (PGD) & 44 & 44 & 42 & 41 & 39 & 38 & 36 & 0 & 0\\
        {} & Smooth-Max-16 (DDN) & 54 & 50 & 46 & 42 & 37 & 33 & 28 & 0 & 0  \\
        {} & Smooth-Mean-16 (DDN) & 55 & 53 & 51 & 50 & 47 & 43 & 38 & 0 & 0\\ 
        {} & Smooth-Max-8 (PGD) & 42 & 41 & 38 & 36 & 33 & 30 & 27 & 0 & 0\\
        {} & Smooth-Mean-8 (PGD) & 45 & 43 & 40 & 39 & 37 & 35 & 33 & 0 & 0\\
        {} & Smooth-Max-8 (DDN)  & 54 & 49 & 44 & 39 & 35 & 31 & 25 & 0 & 0\\
        {} & Smooth-Mean-8 (DDN) & 54 & 53 & 50 & 46 & 42 & 37 & 34 & 0 & 0\\
        {} & SmoothAdv~\cite{salman2019provably} & 37 & 33 & 29 & 26 & 23 & 18 & 16 & 0 & 0\\
        \bottomrule
    \end{tabular}
    
    \label{appdxt:add_cifar10_results}
\end{table}
\clearpage

\noindent\textbf{Random Sampling versus Dense Sampling.} In order to understand the effect of sampling, we analyse the performance of Smooth-Reduce on CIFAR-10 under two sampling schemes: (1) randomly sampling patches under an Uniform distribution, and, (2) densely sampling patches with a specific stride length. Intuitively, the two sampling schemes should not affect performance given enough number of patches. \Figref{appdxfig:cifar10_random_dense} shows that this conjecture holds, with both Smooth-Max and Smooth-Mean presenting comparable performance under the two sampling schemes.

\begin{figure}
    \centering
    \includegraphics[width=0.5\linewidth]{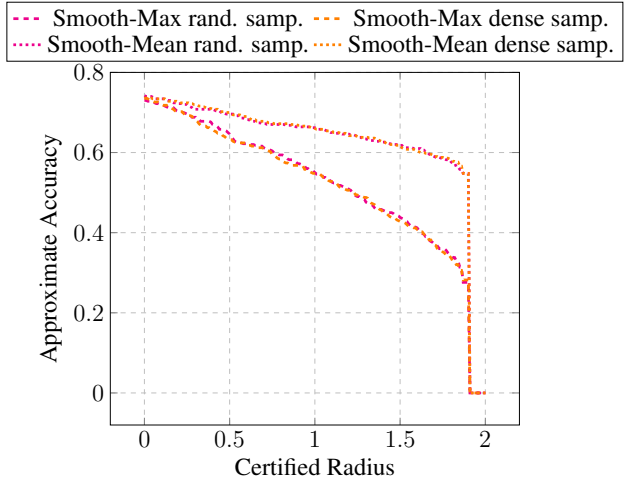}
    \caption{\textbf{Effect of Sampling Algorithm } The sampling algorithm does not affect the certified accuracy in any significant manner for both Smooth-Reduce classifiers, suggesting that the only hyperparameter of consequence is the number of patches ($k$).}
    \label{appdxfig:cifar10_random_dense}
\end{figure}

\noindent\textbf{Performance with the same number of inferences.} An important question that arises is if the improved certification performance is an artifact of the higher number of samples. We show that this is not the case by certifying SmoothAdv~\cite{salman2019provably} and Smooth-Reduce with the same number of samples, $N=100k$. For ensuring fair comparison, we reduce the failure rate probability rate per Smooth-Reduce sub-classifier to $\alpha=0.01$ in comparison to $\alpha=-0.001$ for SmoothAdv. As we observe in \Figref{appdxfig:conf_bounds_full}, we achieve higher certified accuracies as well as better certified radii, given the same amount of compute. 
\begin{figure}
    \centering
    \includegraphics[width=0.85\linewidth]{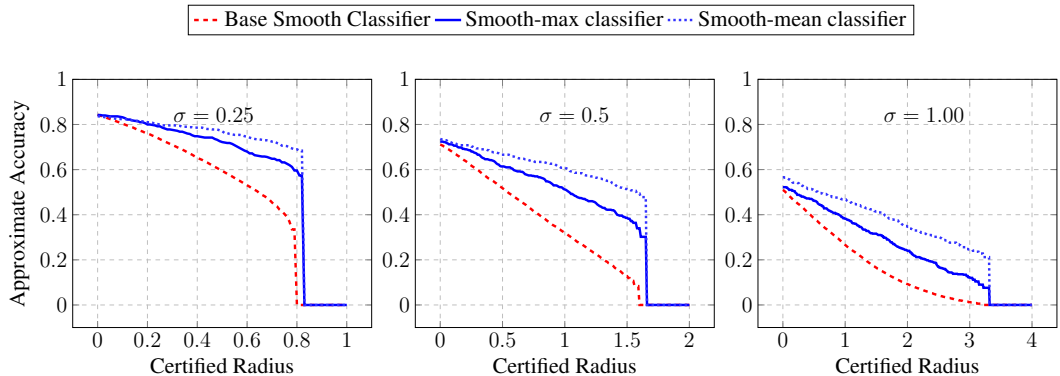}
    \caption{Confidence calibrated Smooth-Reduce}
    \label{appdxfig:conf_bounds_full}
\end{figure}

\section{Video Classifiers}
\label{appdx:sec_vid_classifiers}

Video classifiers come in a large variety of flavors; 3d convolutional~\cite{}, hybrid conv-LSTM models~\cite{}, optical flow-based models~\cite{}, and others. In this paper, we only focus on certifying pure RGB frame based models. This is both due to the models being less computationally expensive as well as achieving high benign performance without a large amount of heuristic tuning. We specifically use the RGB ResNext-101 models from \cite{crasto2019mars} for certifying UCF-101 videos. \cite{crasto2019mars} propose a hybrid RGB-optical flow model as well, which we propose can be adapted easily to a wide variety of video classification tasks. They train two ResNext-101 modesl with 3D convolutions, the first on RGB frame chunks, and, the second on optical flow representations. We just use the first model for certification. However, randomised smoothing for such jointly trained multi-model classifiers is a separate and interesting technical discussion in itself. 

\noindent\textbf{Training.} For training, we initialize our ResNext-101 with weights from the model in~\cite{crasto2019mars} pretrained on the Kinetics dataset. We then train the 3D CNN with $8$ or $16$ frame chunks from the UCF-101 training set. Following \cite{crasto2019mars}, we use SGD with weight decay of $0.0005$, momentum of $0.9$, and initial learning rate of $0.1$. In order to make the classifiers robust to Gaussian noise, we also use Gaussian noise augmentation similar to \cite{cohen2019certified}. Further, we also use the noise-variance scheduling scheme presented in \cite{salman2019provably}, by slowly incrementing noise from $0$ to the required noise levels every $20$ epochs. For inference, the video classifier averages the logits of non-overlapping $8$ or $16$ frame chunks sequentially sampled from the video stream. A pictorial depiction can be seen in \Figref{fig:video_block}.
We used the first train split and the first test split for training and testing our model, respectively. Our base model achieves $\sim86\%$ benign accuracy on the testset. 

For Smooth-Reduce prediction, we follow the procedure presented above in \Secref{subsec:video_clf} by modifying the inference step. We first sample $k$ overlapping sub-videos from the original test video-stream. For our experiments, we use $64$ and $128$ frame subvideos. Next, we create $n$ copies for each sub-video and run the base video inference described above with $16$ or $8$ frame chunks for each noisy copy. The predictions are then aggregated over the copies using the selected \textsc{Aggregate} (\emph{max}/\emph{mean}) Smooth-Reduce methods. The algorithm then returns the class with the largest count. We show results of this in \Figref{appdxfig:vid_results} for noise varinaces of $0.25, 0.5,$ and $1.0$. Notice that while the certified radii are still somewhat low, Smooth-Reduce outperforms standard Randomized smoothing, certifying not only larger radii but also providing greater certified accuracies. We also see that higher chunk sizes and sub-video sizes result in better certification performance in terms of certified accuracy.
 
\begin{figure}[htp]
    \centering
    \includegraphics[width=0.99\linewidth, trim={3cm, 8cm, 3cm, 6cm}, clip]{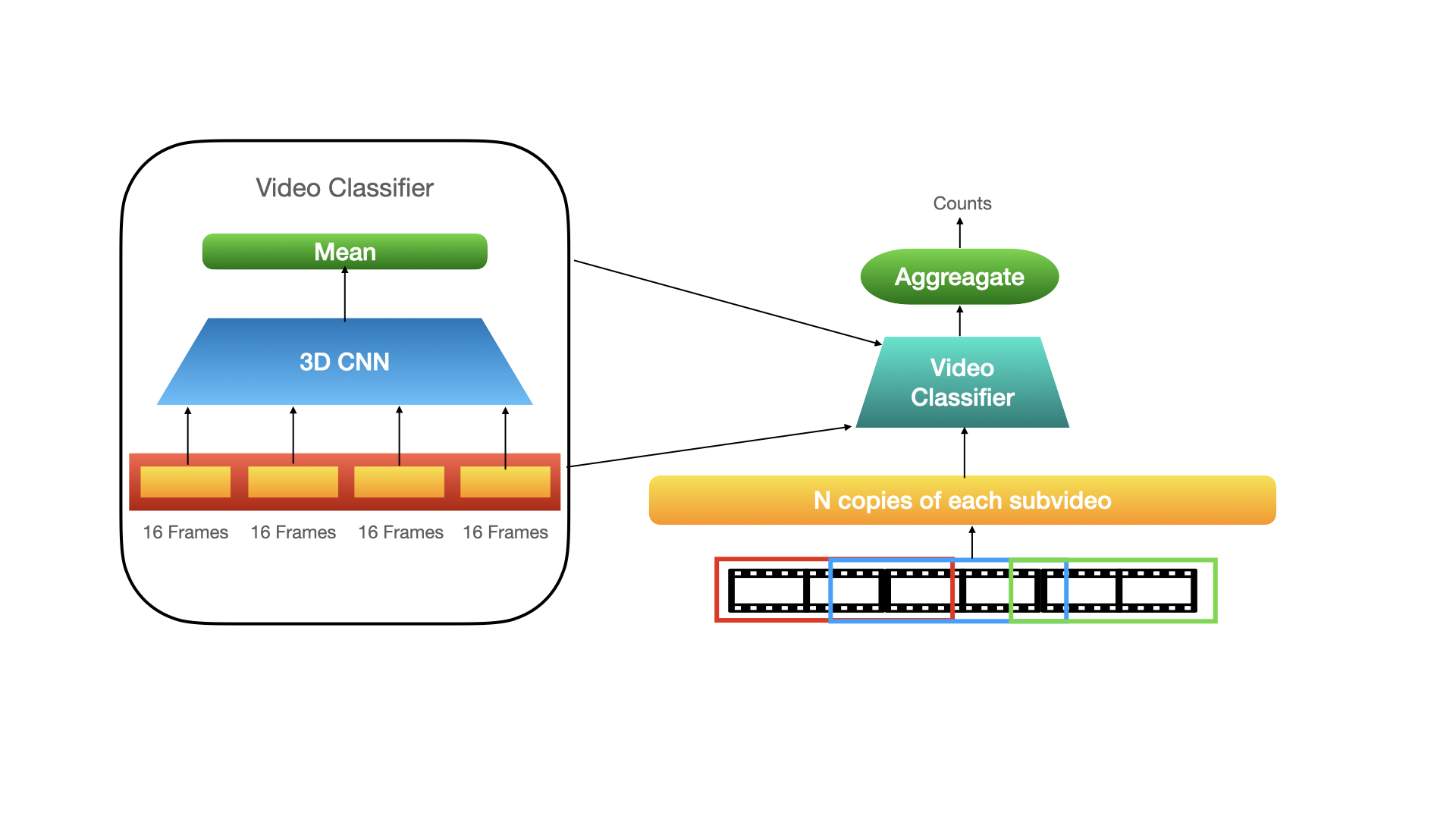}
    \caption{\textbf{Smooth-Reduce for Videos:} Video classifiers include averaging over frames or chunks of frames. Observing that larger chunk sizes provide better certificates, Smooth-Reduce takes this a step further by first sampling overlapping sub-videos with $4$ or $8$ chunks of $16$ frames each. We then aggregate the smooth predictions over sub-videos.}
    \label{fig:video_block}
\end{figure}
\begin{figure}[hb]
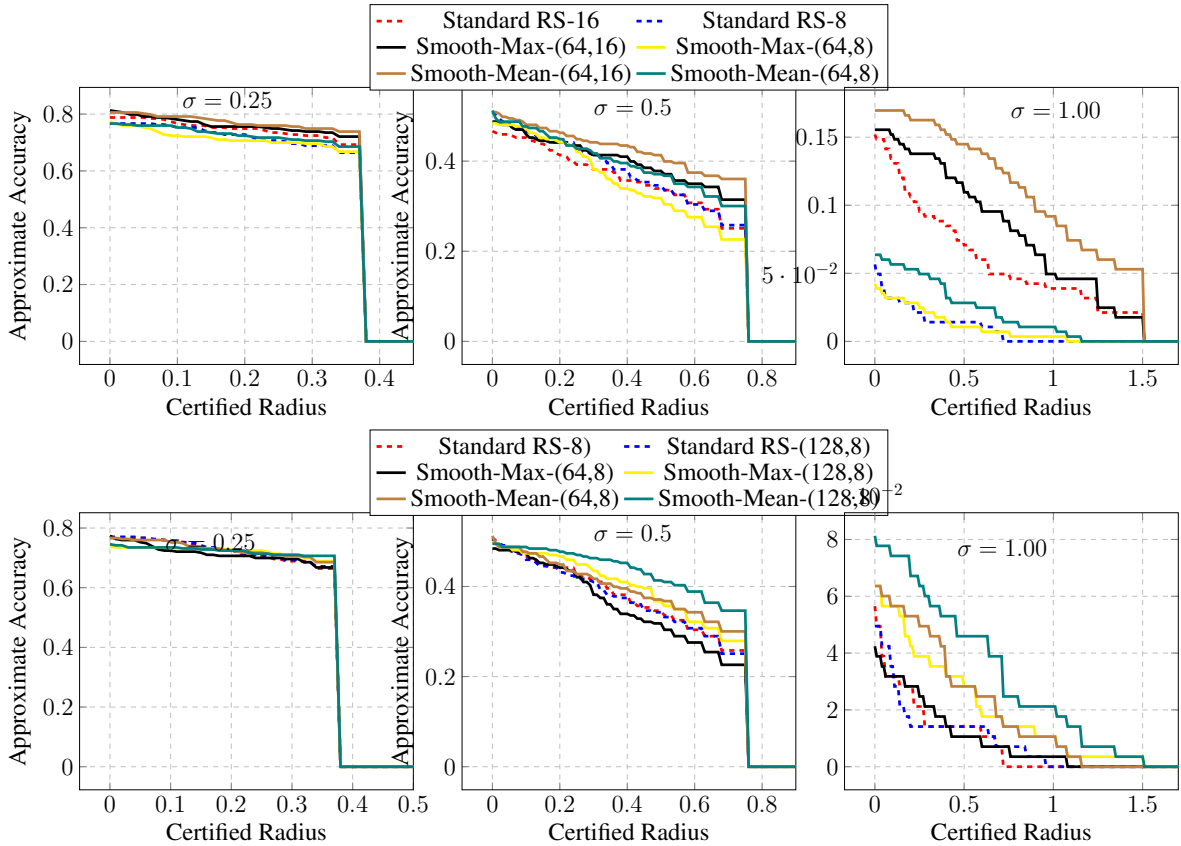

    \centering
    \begin{tabular}{c}
    \includegraphics[width=0.95\linewidth]{Plots/app_ucf101_ablation_chunk_size.tex} \\
    \includegraphics[width=0.95\linewidth]{Plots/app_ucf101_ablation_subvideo_size.tex}
    \end{tabular}
    \caption{\textbf{Additional video results.} Notice that Smooth-Mean certifies larger radii while presenting higher certified accuracies. Another point of interest is that larger sub-videos and larger chunks show better certification performance. Model nomenclature is as follows; for standard randomized smoothing, models are named as \emph{Standard RS}-\textsc{chunk-size}; for Smooth-Reduce, we use \emph{Smooth-\{max/mean\}}-\textsc{sub-video size, chunk size}, in terms of number of frames.}
    \label{appdxfig:vid_results}
\end{figure}
\clearpage

\section{Deferred Theorems and Proofs}
\label{appdx:sec_thms}

\begin{theorem}[Smooth-Max classifiers certify larger radii]
    Let $\hat{f}$ and $\bar{f}$ be the standard RS classifier and the Smooth-Max classifier as defined, if $R_{\hat{f}}$ and $R_{\bar{f}}$ represent the certified radii for the two classifiers for a specific input, $\rvx$, then
    \begin{align*}
        R_{\bar{f}} &\geq R_{\hat{f}}
    \end{align*}
    for all $\rvx$.
    \label{thm:improved_radii}
\end{theorem}
\begin{proof}
            Assume that the correct class predicted by both $\hat{f}$ and $\bar{f}$ is $A$. Let $\underline{p_A}$ be the probability estimate for the $\hat{f}$ smooth classifier, and  $\underline{p_A'}$ be that for $\bar{f}_4$ for $n$ samples. 
            
            We make a minor assumption here, that the input set is large enough that it contains the original image, $\rvx_c = \rvx$. This can be easily enforced by appending the original image to the input set. 
            
            Now, 
            \begin{eqnarray}
                    \underline{p_A} &= \frac{1}{n} \sum_{j=1}^n f(\rvx_c + \rvz_j) \\
                    \underline{p_A'} &= \frac{1}{n} \sum_{j=1}^n \max_i f(\rvx_i + \rvz_j) 
            \end{eqnarray}, where $\rvx_c$ refers to the center-crop of the resized $\rvx$.
            
            By definition, $\underline{p_A'}$ will always be greater than or equal to $\underline{p_A}$. Therefore, the above statement holds true.
\end{proof}            

We also prove that Smooth-Reduce classifiers have a lower failure probability for a given perturbation $\delta$.

\begin{theorem}[Smooth-Reduce confidence bounds]
    Let $\hat{f}$ and $\bar{f}$ be the smooth and Smooth-Reduce classifiers defined above. Let $f_i$ be the sub-classifiers in $\bar{f}$. Let $R$ be the certified radius for $\bar{f}$ w derived using the Smooth-Reduce \textsc{Certify} subroutine with $n$ samples and $k$ patches, with probability $\alpha_1$. Let $R_i,~i=1:k$ be the same for the sub-classifiers, $f_i$, derived using standard smoothing certification with $n$ samples with probability $\alpha$. Then, for Smooth-Mean classifiers, $\alpha_1 \leq  e^{-k\alpha}(2e\alpha)^{k/2}$
    \label{thm:conf_bounds}
\end{theorem}
\begin{proof}

    
    
    

    
    

    Assume that our Smooth-Mean classifer, $\bar{f}$ CERTIFY method returns some certified radius, $R$ with the correct class, $A$ for the given number of samples, $N$ and patches, $p$. Further, we can use \textsc{CERTIFY} from \cite{cohen2019certified} to estimate certified radii, $R_i$, for each of the subclassifiers, $f_i$ in $\bar{f}$. We assume here that the hard-classifier ensemble and the soft ensemble (that Smooth-Mean uses) are equivalent. Under this assumption, as Smooth-Mean relies on majority vote, in order for $\rvx + \delta$ to be an adversarial example, we need at least half of the classifiers to fail. To analyse this, let $m_i$ be a Bernoulli random variable such that it takes the value $1$ if classifier $f_i$ fails and $0$ otherwise. 
    
    Thus, 
    \begin{align*}
         \sP [\|\delta\| < R_i] = \sP[m_i = 1] = \alpha
    \end{align*}
    
    Therefore, for $\rvx + \delta$ to be an adversarial example, 
    \begin{align*}
        \sP [ \|\delta\| < R] &= \sP[\sum_{i=1}^k m_i \geq k/2]\\
    \end{align*}
    Using a Chernoff bound~\cite[Thm. 2.3.1]{vershynin2018high} for the sum of independent Bernoulli random variables, we get;
    \begin{align*}
        \sP[\sum_{i=1}^k m_i \geq k/2] \leq e^{-k\alpha}(2e\alpha)^{k/2}
    \end{align*}
    
    Note that this function decays very quickly with $k$, and therefore can be easily tuned to get better confidence bounds.  
\end{proof} 

While our approach relies on analysing a specific version of the adversarial example which attacks all classifiers simultaneously, we recognize that this might not be the case in general. For example, another attack may presume to make the classifier abstain every time. We do not analyse this case here, and leave the details to future work.

\subsection{Analysing Logits under Smooth-Reduce Ensembling}

We further validate our claims regarding confidence intervals of Smooth-Reduce certificates by analysing the logit distribution for standard RS and Smooth-Reduce classifiers. 

\noindent\textbf{Setup:} We study the distributions of logits for the most probable and the second most probable class for standard RS and Smooth-Reduce classifiers. For this, we consider a few test datapoints for both images, and videos and certify the best SmoothAdv classifier. Further, we certify both Smooth-Max and Smooth-Mean classifiers under the same setup. We then plot histograms of the distributions of logits. \Figref{appdxfig:logit_cifar10} and \Figref{appdxfig:logit_vids} show exemplars of generated histograms. 

\noindent\textbf{Observations and Inferences.} Notice that the certified radius, $R$ from \Eqref{eq:smooth_radius} is proportional to the difference in the estimated probabilities of the two most probable classes. This difference is also proportional to the classifier margin. Therefore, in order to get better certificates, we need to ensure that the smooth-classifier presents large margins, as well higher probability estimate for the true class, $c_A$. Also, in order to reduce abstentions, $\underline{p_A} > 1/2$ and its variance must be low. 

We see that Smooth-Max and Smooth-Mean outperform SmoothAdv on both these criterion in \Figref{appdxfig:logit_cifar10}. Notice here that $R$ depends on the difference between the means of the distribution for the most probable class (blue) and the second most probable class (orange). We see that while Smooth-Max outperforms SmoothAdv in terms of the overall proabability estimate, the margin itself is not improved much. This may lead to higher abstention rates as well as lower certificates. However, Smooth-Mean showcases not only higher estimates of $\underline{p_A}$ but also a lower variance, thus improving upon both the certified radius and probability of abstention.

\begin{figure}[htp]
    \begin{tabular}{c c c} 
        \includegraphics[width=0.33\linewidth]{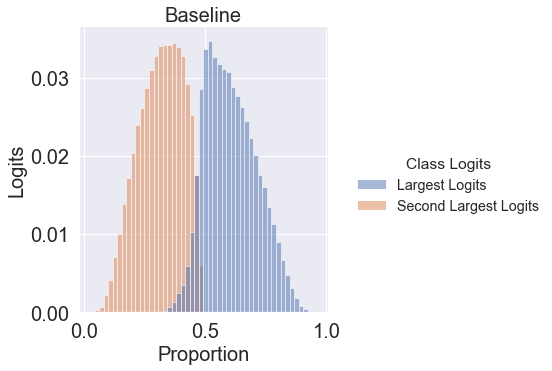} & \includegraphics[width=0.33\linewidth]{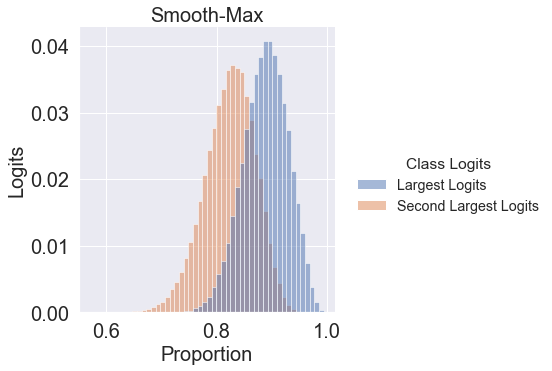} &
        \includegraphics[width=0.33\linewidth]{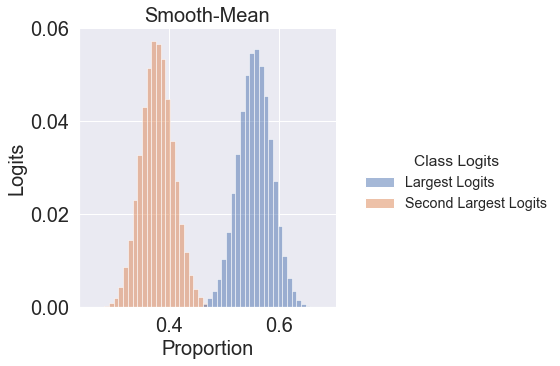}
    \end{tabular}
    \caption{\textbf{Logit Distributions for Smooth Classifiers for CIFAR-10.}\small \sl The histograms are arranged as follows: (L) SmoothAdv classifier,  (M) Smooth-Max classifier, and (R) Smooth-Mean classifier. The blue bars the logits for the most probable class, while the orange represent those for the second-most probable class. For good smooth classifiers, the blue peak should be at $1.0$ with low variance and the orange peak should be close to $0$. Observe that Smooth-Max is performs better than SmoothAdv on the first criterion, while Smooth-Mean performs better on both.}
    \label{appdxfig:logit_cifar10}
\end{figure}

For an exemplar certificate in the case of video classifiers, we immediately observe similar behavior. In \Figref{appdxfig:logit_vids}, we observe logit distributions for two examples from the UCF-101 dataset. It is clearly evident that while Smooth-Max and Smooth-Mean provide better margins and lower variance than standard randomized smoothing. However, the logit values are still skewed lower than those for images, and the variance across the logit values is fairly higher. This explains why our video certificates are far lower than image certificates. We conjecture that this is an effect of the difficulty in training noise-robust 3D CNN models for video. However, we leave exploring this phenomenon to future work. 

\begin{figure}[htp]
    \begin{tabular}{c c c} 
        \includegraphics[width=0.33\linewidth]{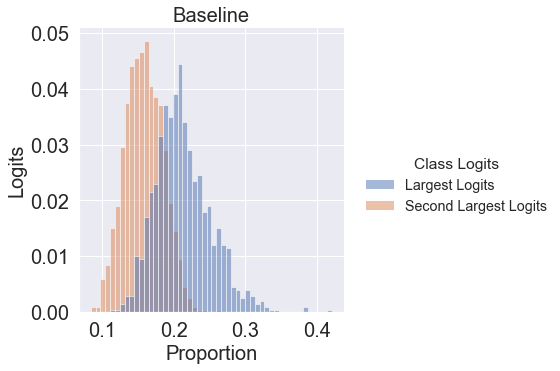} & \includegraphics[width=0.33\linewidth]{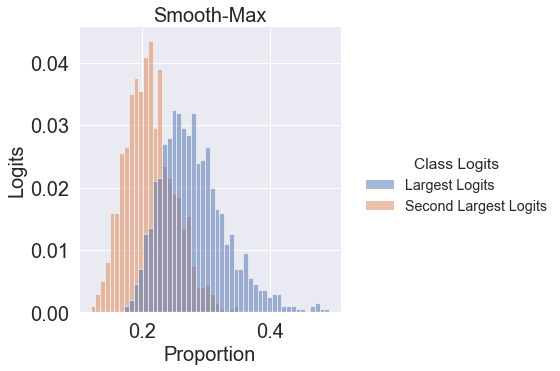} &
        \includegraphics[width=0.33\linewidth]{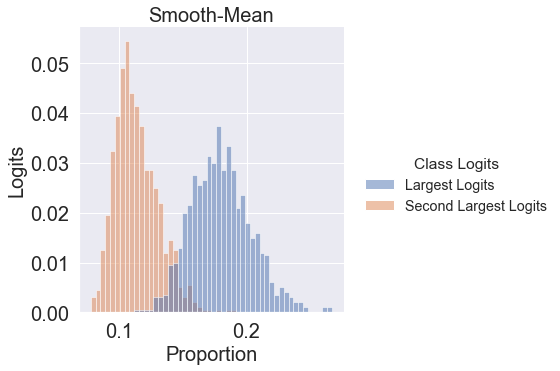} \\
        \includegraphics[width=0.33\linewidth]{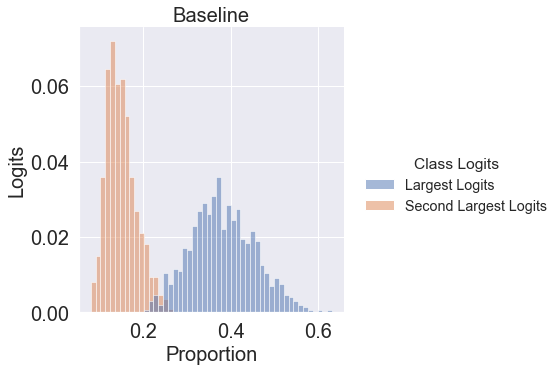} & \includegraphics[width=0.33\linewidth]{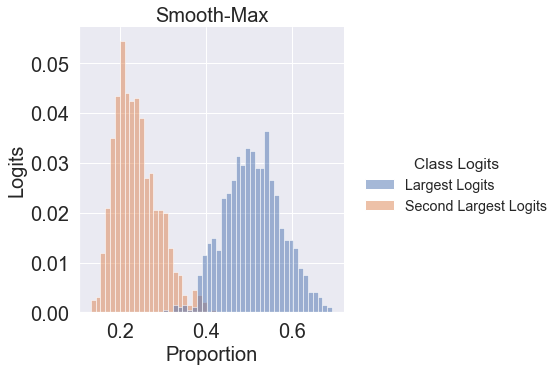} &
        \includegraphics[width=0.33\linewidth]{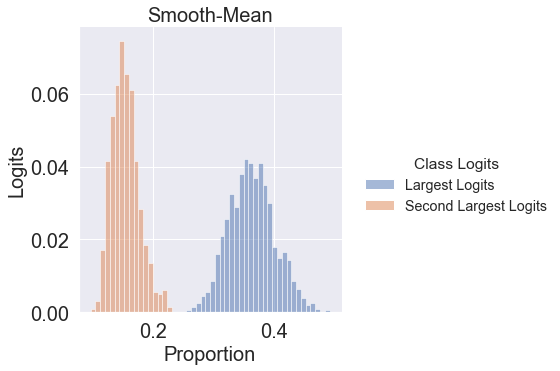}
    \end{tabular}
    \caption{\textbf{Logit Distributions for Smooth Classifiers for UCF-101.} \small \sl (L) shows logit distributions for standard randomized smoothing, (M) and (R) show the same for Smooth-Max and Smooth-Mean respectively. Observe that in comparison with image classifiers, RS and Smooth-Reduce with video classifiers leads to lower $\underline{p_A}$ estimates as well as high variance. This results in lower certified radii and higher abstention rates as seen in \Figref{fig:vidresults}. However, Smooth-Mean still outperforms SmoothAdv.}
    \label{appdxfig:logit_vids}
\end{figure}

\subsection{Some Additional Discussion on Confidence Intervals for Ensembling}
We also reproduce some results by \cite{Horvath2021BoostingRS} which support increasing success rates for Smooth-Mean. 

Horvath~\etal~\cite{Horvath2021BoostingRS} analyse the following soft ensemble classifier, 
\begin{align*}
        \bar{f}(\rvx) = \frac{1}{k}\sum_{i=1}^k f_i(\rvx) .
\end{align*}
Let $\rvy_i$ be the logits from each of the sub-classifiers, $f_i$, and $\rvy$ be the same for the ensembled classifier. They further model $\rvy_i=\rvy_{i, c} + \rvy_{i, p}$,  where $\rvy_{i,c}$ is a random variable representing the contribution of the $i^{th}$ sub-classifier and $\rvy_{i,p}$ represents the contribution due to random noise added during randomized smoothing. Further they assume, $\E [\rvy_{i,c}] = \rvc$ and $\E[\rvy_{i,p}] = 0$. The variance of $\rvy_p$ is assumed to be $\Sigma_p$ where $\Sigma_{ii} = \sigma_i^2$, and $\Sigma_{ij} = \sigma_i\sigma_j\rho{ij}$. This holds as the two processes of training and smoothing are independent. Notice that Smooth-Mean classifiers are a special class of such classifiers, where the sub-classifiers are constructed with independent sampling matrices. 

Now, they analyse the class margins, $t_i = y_1 - y_1$ where $y_i$ are elements of $\rvy$ and $1$ is the majority class (WLOG). Notice, 
\begin{align*}
    \E [\rvz_i] = c_1 - c_i \\
    \text{Var}[\rvt_i] =  \sigma_{p, 1}^2 + \sigma_{p, i}^2 +\sigma_{c,1}^2 + \sigma_{c,i}^2 -  2\sigma_{p,1}\sigma_{p,i}\rho{p1,i} - 2\rho{c,1i}\sigma+{c,1}\sigma{c,i}
\end{align*}

Through careful arithmetic, they show that, 
\[
\text{Var}(\bar{\rvt})  = \sigma_p^2(k) + \sigma_{c}^2(k), 
\]
where $\sigma_m^2 = \frac{k + {k \choose 2}\zeta_m}{k^2}(\sigma_{p,1}^2 + \sigma_{p,i}^2 -2\rho_{p,1i}\sigma{p,1}\sigma_{p,i}$ for $m\in[p,c]$, and $\zeta_m \in  [0,1]$ refers to parameter denoting covariance between $y_{i,m}$ and  $y_{j,m}$; refer ~\cite{Horvath2021BoostingRS} for more details.

This decoupling of the variance between the perturbations due to RS and ensembling proves to be important in understanding the benefits of ensembling. They present the following result on success probabilities, 

\noindent\textbf{Informal Theorem}[From \cite{Horvath2021BoostingRS}]
    \textit{For a soft-ensemble of $k$ classifiers which provides a certificate with radius $R$ with probability $1 - \alpha_1$, the upper bound of the probability of failure decreases with $O(k^2)$}

To measure the effect on success probability, we consider the probability of a majority of the sub-classifiers predicting class $1$, $\beta_1$.
\[
    \beta_1 = \sP(\bar{f}(\rvx+\rvz) = 1) \\
    = \sP(\bar{\rvt} > 0 :\forall i\in[2, C]) \\
    = \int_{\bar{\rvz} > 0 : \forall i \in [2, C]}\sP(\bar{\rvt}).d\bar{\rvz} 
\]
While this integral cannot be evaluated directly as we do not know the density function for $\rvz$, we can lower bound $\beta_1$ using Chebyshev's inequality and the union bound over the incorrect $[2, C]$ classes.
\[
    \beta_1 \geq  1 - \sum_{i=1}^C \frac{(\sigma_{i,c}(k)^2 + \sigma_{i, p}(k)^2}{(c_1 - c_i)^2}
 \].
 As $\sigma_{i,c}$ and $\sigma_{i,k}$ decrease quadratically with increasing $k$, we can prove the above theorem.

\end{document}